\documentclass{article}
\usepackage{newpxtext,newpxmath}
\usepackage{PRIMEarxiv}

\usepackage[utf8]{inputenc} 
\usepackage[T1]{fontenc}    
\usepackage{hyperref}       
\usepackage{url}            
\usepackage{booktabs}       
\usepackage{amsfonts}       
\usepackage{nicefrac}       
\usepackage{microtype}      
\usepackage{lipsum}
\usepackage{fancyhdr}       
\usepackage{graphicx}       
\graphicspath{{media/}}     
\usepackage{subcaption}
\usepackage{tabularx}
\usepackage{amsmath, amssymb, amsthm}
\usepackage{tikz}
\usetikzlibrary{arrows.meta, positioning, calc}
\usepackage{xcolor}
\newtheorem{prop}{Proposition}
\usepackage[most]{tcolorbox}
\definecolor{cvnnNavy}{RGB}{14,35,68}
\definecolor{cvnnPurple}{RGB}{103,48,166}
\definecolor{cvnnBlue}{RGB}{31,91,170}
\definecolor{cvnnGreen}{RGB}{45,135,70}
\definecolor{cvnnOrange}{RGB}{220,105,35}
\definecolor{cvnnRed}{RGB}{190,45,30}
\definecolor{cvnnGray}{RGB}{150,150,150}

\title{ When do complex-valued neural networks help? \\
A study of representation, geometry, and optimization
}

\author{
  Ashutosh Kumar \\
  Owl Autonomous Imaging, Inc. \thanks{Work done outside of Owl AI}\\
  \texttt{ak1825@rit.edu}
}

\begin{document}
\maketitle

\begin{abstract}
Complex-valued Neural Networks (CVNNs) are particularly well-suited to domains where information is inherently encoded in both magnitude and phase, making complex-valued representations a natural choice for modeling. However, a complex-valued input alone does not explain when complex arithmetic should improve the model learning. The label signal may reside in amplitude, phase, their coupling, or a symmetry that a real-valued model could also represent under the right coordinate system. This paper studies precisely this question through a representation-first approach towards evaluating CVNNs against Cartesian real, polar, phase-only, magnitude-only, parameter-matched real, and FLOP-matched real baselines. The results show that the advantage of CVNNs is conditional rather than universal. On synthetic Radio Frequency (RF) data, Phase Shift Keying (PSK)-only tasks favor phase-aware and complex-valued models, Quadrature Amplitude Modulation (QAM)-only tasks favor magnitude, mixed PSK+QAM gives only a small edge to complex-valued representation, and unseen carrier-phase rotations break coordinate-dependent models unless signal augmentation is introduced. The same pattern appears outside RF: in a quantum-wavefunction pilot study, momentum is invisible to $|\psi|$ but recoverable from phase information, while in an EEG analytic-signal pilot study, phase locking, amplitude bursts, and phase-amplitude coupling each select different coordinate views. This work also identifies a methodological artifact in the benchmarking of CVNNs. 

On the RadioML 2018.01A subset~\cite{8267032}, a CReLU~\cite{shang2016understanding} complex-valued model outperformed the best real baseline by \textbf{22.94 percentage points (PP)} under \emph{matched-shared-trial selection}. However, under \emph{independent per-family tuning}, on the same data and the same 16-trial search space, this performance gap collapses to \textbf{2.46 PP}. A deeper analysis of the activation functions shows that the per-step gradient traces the inflated gap to a high learning-rate first-step instability in the real baselines: AdamW~\cite{loshchilov2017decoupled} pushes the classifier head into a dead-ReLU regime, while the complex-valued model tolerates the same trial because its $(\Re,\Im)$ parameter coupling distributes the loss signal across more degrees of freedom. A learning-rate $\times$ activation factorial confirms that the failure is primarily \textbf{hyperparameter-driven, not activation-driven}. This work frames these findings through \emph{Wirtinger calculus}, a \textbf{Liouville activation trilemma}, and a \emph{structural equivalence} between complex 1D convolutions and the $U(1)$-equivariant subspace of stacked-real two-channel convolutions. The resulting claim is scoped as follows: complex-valued models are useful structured inductive biases for phase- and amplitude-bearing data, but their improvements depend on representation, symmetry, and optimization. Therefore, the large CVNN-vs-real gaps should be reported under both matched-shared-trial and independent per-family selection before being interpreted as an architectural advantage. The code is available on GitHub\footnote{\url{https://github.com/ashu1069/Neural-Networks-in-Complex-Spaces}}.
\end{abstract}


\section{Introduction}
\label{sec:introduction}

Many datasets contain phase information that can be represented or measured in complex number spaces: \textbf{(1)} Radio receivers record in-phase and quadrature samples, \textbf{(2)} Fourier and wavelet transforms expose magnitude and phase, \textbf{(3)} quantum states are complex wavefunctions, \textbf{(4)} analytic-signal representations of neural time series separate amplitude envelopes from instantaneous phase, as shown in Figure~\ref{fig:data_types}. These domains make Complex-valued Neural Networks (CVNNs) an intuitive modeling choice, and prior work has shown that complex layers, complex normalization, and complex activations can be competitive with or superior to real-valued counterparts in settings such as signal classification~\cite{hirose2006complex,trabelsi2017deep}, Magnetic Resonance Imaging (MRI)~\cite{virtue2017better, cole2021analysis}, and unitary recurrent modeling~\cite{arjovsky2016unitary, wisdom2016full}.
\begin{figure}[hbt!]
    \centering
    \includegraphics[width=\linewidth]{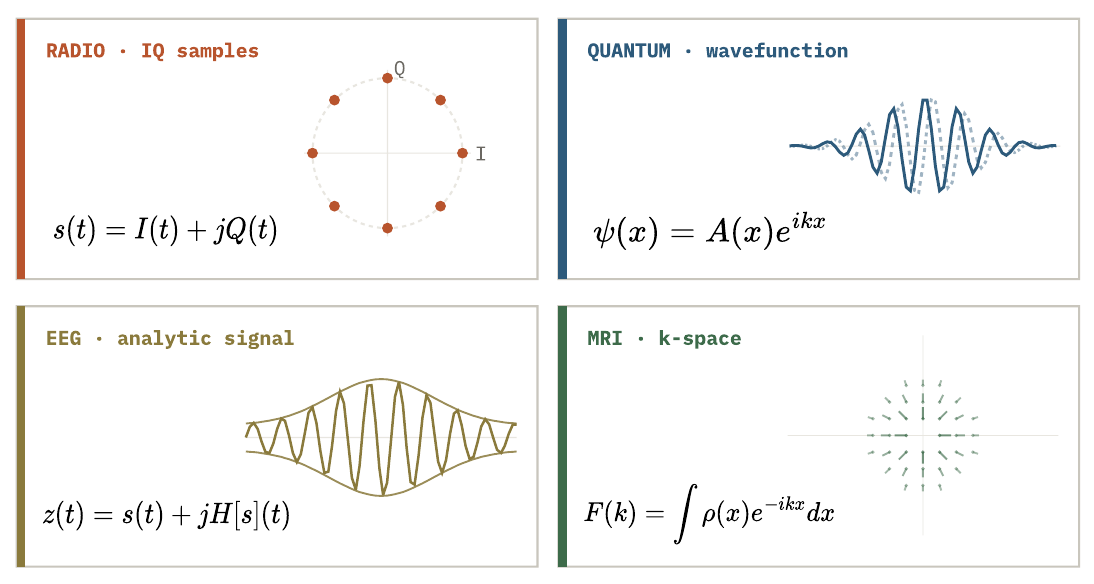}
    \caption{\textbf{Complex-valued data}. Complex-valued signals arise across radio communications, quantum mechanics, neural recordings, and medical imaging. The mathematical object is identical; what differs is which axis of $z = re^{i\theta}$ carries the task-discriminative information.}
    \label{fig:data_types}
\end{figure}
However, the fact that the input data is complex-valued does not by itself explain why a complex-valued network should help. A sample $z = x + iy = re^ {i\theta}$ contains several distinct kinds of information, which we will discuss further in this paper. The magnitude $r = |z|$ may encode energy or amplitude; the phase $\theta = \arg z$ may encode angle, delay, carrier convention, or momentum; and the task may depend on a relation between amplitude and phase. A real-number-valued model given $(x,y)$ sees the full Cartesian representation, and given $(r,\cos\theta,\sin\theta)$, it sees magnitude and phase as two explicit entities. A magnitude-only or phase-only model exposes controlled information bottlenecks with real-valued models. Conclusively, a measured advantage of CVNNs can be attributed to several mechanisms: the model may receive a better coordinate system, impose a useful symmetry, regularize the wrong degrees of freedom, or simply tolerate a hyperparameter regime in which real-number model baselines fail.

This paper explicitly asks a narrower question than "\textit{Are complex-valued networks better than real-valued ones?}'': \textbf{When do complex-valued networks dominate, and what mechanism explains the performance gain?} Additionally, the author posits that CVNNs should be treated as structured inductive biases for phase- and amplitude-bearing data, not as universally stronger function approximators. They are most useful when the geometry of the task aligns with the multiplication operations in the complex space. Their performance can be matched or surpassed when explicit real coordinates expose the relevant variable, and large CVNN-vs-real performance gaps can be artifacts of the evaluation protocol rather than genuine superiority in the architecture.

Structurally speaking, the reason complex layers can help is that complex multiplication is not an arbitrary two-channel real number mapping. Multiplying $(x+iy)$ by $(a+ib)$ gives
\[
    (a+ib)(x+iy) = (ax-by) + i(ay+bx),
\]
which corresponds to the real number matrix
\[
    \begin{pmatrix}
        a & -b \\
        b & a
    \end{pmatrix}
    \begin{pmatrix}
        x \\
        y
    \end{pmatrix}
\]
Now, let's understand that this is a rotation-and-scale coupling between the real and imaginary channels. Whereas real-valued layers can learn arbitrary $2 \times 2$ channel mixing or couplings, while the complex-valued layers impose the \textit{structural constraint} that each coupling must be restricted in the two-parameter subspace, $aI+bJ$. The paper further shows that this subspace is exactly the $U(1)$-equivariant subspace of stacked-real 2-channel convolutions. Complex multiplication, therefore, implements a built-in \textit{symmetry constraint}, restricting real-valued $2\times 2$ channel-mixing to the structured form that is automatically preserved in complex coordinates. This connects CVNNs to the broader group-equivariant learning perspective~\cite{cohen2016group, cohen2016steerable, bronstein2021geometric, weiler2019general, worrall2017harmonic}, while keeping the empirical question concrete: \textbf{the constraint should help only when the task respects the relevant phase geometry}.

Next, concerning the optimization in complex space, CVNN comparisons often use matched hyperparameter trials to assert a similar ground for evaluation of real and complex families. However, a shared trial can answer a different question from independent per-family tuning. A matched-shared-trial rule looks to find out which family is more robust under a common hyperparameter allocation, whereas independent per-family selection asks which family achieves the best-tuned performance. These are both legitimate ways of approaching the evaluation, but they should not be combined with each other. In the experiments with the RadioML 2018.01A dataset~\cite{8267032}, this distinction changes the ultimate conclusion: under CReLU~\cite{shang2016understanding} activation, the complex-valued model surpasses the best real-valued baseline by 22.94 percentage points (PP) under matched-shared-trial selection, but the gap collapses to 2.46 PP under independent per-family selection (same data and the same 16-trial search space). Further investigations lead to the per-step gradient telemetry, which shows that the larger gap is caused by a high-learning-rate first step of the AdamW optimizer, which drives real baselines into a dead-ReLU regime, while the complex model tolerates the same configuration.

This work, therefore, separates three mechanisms that are often entangled in CVNN claims:
\begin{enumerate}
    \item \textbf{Representation:} whether the label signal lives in magnitude, phase, Cartesian coordinates, or a phase-amplitude relation.
    \item \textbf{Symmetry:} whether complex multiplication supplies an appropriate $U(1)$-equivariant constraint for the task.
    \item \textbf{Optimization and selection:} whether an apparent gap reflects peak performance or robustness to a shared hyperparameter regime.
\end{enumerate}

We evaluate these mechanisms through a set of controlled stress tests. On synthetic Radio Frequency (RF) modulation tasks, Phase Shift Keying (\texttt{PSK})-only data favor phase-aware and complex models, Quadrature Amplitude Modulation (\texttt{QAM})-only data favor magnitude, mixed \texttt{PSK+QAM} gives a smaller edge to CVNN, and fixed carrier-phase rotations expose the lack of \emph{automatic invariance} in the tested architectures. We then performed pilot studies from two distinct domains: (1) in a quantum wavefunction, the phase is indispensable for momentum classification - magnitude-only models remain at chance while phase-aware and complex views solve the task; (2) for an EEG analytic-signal, different event types select different coordinate views - phase locking is phase-dominant, amplitude bursts are magnitude-dominant, and phase-amplitude coupling favors a joint polar representation. These experiments support the same conclusion across domains: the winning model is determined less by whether the input is formally complex than by where the task-relevant information lives.

The paper makes four contributions:
\begin{itemize}
    \item \textbf{A representation-first protocol for CVNN evaluation.} We compare native complex networks with Cartesian real, polar, phase-only, magnitude-only, parameter-matched real, and FLOP-matched real baselines. This protocol separates complex arithmetic from coordinate access and information bottlenecks.

    \item \textbf{A structural interpretation of complex convolution.} We show that one-dimensional complex convolution is equivalent to a $U(1)$-equivariant two-channel real convolution whose kernel taps are constrained to the $aI+bJ$ subspace. This identifies the complex layer as a symmetry-constrained real layer rather than a mysterious additional modeling primitive.

    \item \textbf{Cross-domain stress tests of phase and amplitude information.} Synthetic RF, quantum wavefunctions, and EEG analytic signals all show that phase-bearing, amplitude-bearing, and phase-amplitude-coupled tasks select different representations. Complex-valued models help when their inductive bias matches this structure, but they do not dominate all coordinate views.

    \item \textbf{A selection-rule caution for CVNN-vs-real benchmarking.} On a controlled RadioML 2018.01A subset, we show that a large matched-shared-trial CVNN gap can shrink by an order of magnitude under independent per-family tuning. Gradient telemetry and a learning-rate/activation factorial trace the inflated gap to a high-learning-rate dead-seed mechanism in the real baselines.
\end{itemize}

The rest of the paper is organized as follows. Section~\ref{sec:background} develops the first-principles background: complex coordinate views, Wirtinger calculus, the Liouville activation trilemma, and the $U(1)$-equivariant interpretation of complex convolution. Section~\ref{sec:protocol} defines the representation-first evaluation protocol and the two selection rules. Section~\ref{sec:results_representation} presents the RF, quantum, and EEG stress tests. Section~\ref{sec:radioml} analyzes the RadioML result and the selection-rule artifact. Section~\ref{sec:related_work} positions the paper in CVNNs, equivariant learning, signal-processing applications, and evaluation methodology. Section~\ref{sec:limitations} states the limitations, and Section~\ref{sec:conclusion} concludes this experimental study of CVNNs.
\section{Related Work}
\label{sec:related_work}
CVNNs share a fair amount of history in signal processing and deep learning. Hirose's textbook~\cite{hirose2006complex} gives the standard mathematical characterization of complex-valued model learning, while Aizenberg's multi-valued neuron structure~\cite{aizenberg2011complex} represents an earlier segment of complex-domain neural models. The modern deep-learning formulation is anchored primarily by Trabelsi et al.~\cite{trabelsi2017deep}, who introduced widely used key components such as complex convolutions, complex batch normalization, complex initialization schemes, and the \texttt{CReLU/ZReLU/ModReLU} activation family. Related work explored complex-valued CNNs for classification tasks~\cite{guberman2016complex} and complex/unitary recurrent networks for stable sequence modeling~\cite{arjovsky2016unitary,wisdom2016full}. This paper builds on similar lines but changes the empirical question. Instead of asking whether a complex-valued model outperforms a real-valued model under a single comparison protocol, the author asks which mechanism explains a measured advantage. The resulting protocol compares native complex networks with Cartesian, polar, phase-only, magnitude-only, parameter-matched, and FLOP-matched real baselines. This separates complex arithmetic from coordinate system access, symmetry constraints, capacity, and optimization robustness.

\paragraph{Wirtinger calculus, complex backpropagation, and activation design.}
The major revelation about CVNNs is that backpropagation (training) does not require holomorphic activations. The standard mechanism uses Wirtinger or CR-calculus, where complex-valued functions are differentiated as real-differentiable maps over real and imaginary coordinates. Brandwood~\cite{brandwood1983complex} introduced the complex gradient operator in adaptive array signal processing, and Kreutz-Delgado~\cite{kreutz2009complex} later gave a tutorial treatment widely used by machine-learning practitioners. Modern autodiff frameworks follow the conjugate-Wirtinger convention for real-valued losses with complex parameters. This calculus makes practical activations such as \texttt{CReLU, ZReLU, ModReLU, Siglog, ComplexCardioid}, and \texttt{ComplexTanh} valid even when they are not holomorphic. Prior CVNN work introduced or used several of these activations~\cite{trabelsi2017deep,arjovsky2016unitary,wisdom2016full}. This paper's contribution is to make the activation trade-off explicit through a \textit{Liouville trilemma} (Proposition~\ref{prop:activation_trilemma}) and empirical Cauchy-Riemann residual diagnostics. This connects the mathematical impossibility of a bounded, entire, non-constant activation to measurable activation behavior at initialization.

\paragraph{Group-equivariant and symmetry-constrained learning.}
A central point of this paper's analysis is that complex multiplication is a symmetry-constrained real map. The broader idea that neural-network layers should respect task symmetries is well established in group-equivariant deep learning. Cohen and Welling introduced group-equivariant convolutional networks for discrete transformations~\cite{cohen2016group}; steerable and general $E(2)$-equivariant CNNs extend this idea to richer transformation groups~\cite{cohen2016steerable,weiler2019general}; and Bronstein et al.~\cite{bronstein2021geometric} survey the geometric deep-learning framework. Harmonic networks use complex circular harmonics to parameterize rotation-equivariant two-dimensional filters~\cite{worrall2017harmonic}. The result is the one-dimensional $U(1)$ specialization of this principle. A complex \texttt{Conv1d} layer is equivalent to a stacked-real two-channel convolution whose kernel taps are constrained to the $aI+bJ$ subspace. Thus, the complex layer is not more expressive than a generic real two-channel layer; it is less expressive in a symmetry-aligned way. This interpretation helps explain why complex layers can help on phase-bearing data while remaining unnecessary or suboptimal when an explicit real coordinate view exposes the relevant variable.

\paragraph{Complex-valued networks for RF modulation classification.}
RF modulation classification is a natural testbed for CVNNs because IQ samples are intrinsically complex-valued. O'Shea et al.~\cite{o2016convolutional} introduced the RadioML 2018.01A benchmark and reported strong real-valued ResNet-style baselines on the full 24-class task. Subsequent work has studied complex-valued architectures for automatic modulation classification, typically reporting modest gains over comparable real-valued models~\cite{krzyston2021modulation,tu2020complex}. The RadioML result (Section~\ref{sec:radioml}) is consistent with this literature once activation and selection effects are separated. Under stable activations and independent per-family tuning, the complex advantage is only a few percentage points. The much larger \texttt{CReLU} matched-shared-trial gap in the experiments is not a contradiction of prior work; it is a selection-rule artifact produced by high-learning-rate instability in the real baselines. The contribution is therefore methodological: \textit{large matched-trial CVNN gaps should be interpreted only after checking independent per-family selection, dead-seed counts, and early-step gradient telemetry}.

\paragraph{Complex-valued models in physics and biomedical signals.}
Complex-valued representations also arise in scientific and biomedical domains. In MRI, CVNNs have been used for fingerprinting and reconstruction, with prior work reporting benefits over real-valued alternatives~\cite{virtue2017better,cole2021analysis}. In physics-informed learning, real-valued neural networks are commonly used to approximate solutions of differential equations~\cite{raissi2019physics}, but many physical states, including quantum wavefunctions, are naturally complex-valued. Our quantum pilot is a small representation test in this setting: momentum is unrecoverable from magnitude alone but recoverable from phase. Biomedical time-series analysis gives another example. EEG deep learning has largely used real-valued networks over raw signals or spectral features~\cite{schirrmeister2017deep,lawhern2018eegnet}. At the same time, analytic-signal methods based on Hilbert or wavelet transforms are standard tools for studying phase locking, amplitude envelopes, and phase-amplitude coupling~\cite{tort2010measuring}. The EEG pilot (Section~\ref{sec:eeg_representation}) uses this signal-processing structure to test the representation-first protocol: phase locking favors phase features, amplitude events favor magnitude, and phase-amplitude coupling favors joint polar information. The pilot is not intended as a clinical benchmark; it shows that the same representation logic transfers beyond RF.

\paragraph{Benchmarking methodology and evaluation variance.}
The selection-rule artifact identified by this work belongs to a broader literature showing that apparent model advantages can shrink or disappear under tighter evaluation protocols. Henderson et al.~\cite{henderson2018deep} documented instability and reporting issues in deep reinforcement learning; Lucic et al.~\cite{lucic2018gans} showed that GAN comparisons are highly sensitive to tuning and evaluation choices; and Melis et al.~\cite{melis2017state} showed that language-model comparisons can be misleading without careful hyperparameter control. Bouthillier et al.~\cite{bouthillier2021accounting} formalized the distinction between seed variance and hyperparameter variance and recommended reporting both.

The contribution to this methodology literature is specific to CVNN benchmarking. Matched-shared-trial selection is often used as a fairness device because it holds the trial index fixed across model families. This paper shows that this rule can instead become a robustness stress test: it may select a trial that the complex model survives, but real baselines do not. Reporting both matched-shared-trial and independent per-family selection, therefore, changes the scientific interpretation of a CVNN-vs-real gap.
\section{Background and First Principles}
\label{sec:background}

CVNNs are introduced as neural networks whose parameters and activations live in $\mathbb{C}$. For this paper, the more important point is not the data type itself, but the structure it imposes on mathematical operations. A \textit{complex representation} bundles amplitude and phase into a single algebraic entity. A \textit{complex layer} constrains how the real and imaginary channels may mix. A \textit{complex activation} must navigate a different set of analytic trade-offs from ordinary real nonlinearities. This section makes these three facts explicit.

\subsection{Complex inputs contain multiple possible prediction variables}
\label{sec:complex_inputs}

A complex number can be written in Cartesian or polar form:
\[
    z = x + iy = r e^{i\theta},
    \qquad
    r = |z| = \sqrt{x^2+y^2},
    \qquad
    \theta = \arg z .
\]
These two descriptions contain the same information, but they expose different variables to a learning algorithm. The Cartesian view $(x,y)$ presents the real and imaginary coordinates directly. The polar view $(r,\cos\theta,\sin\theta)$ separates magnitude from phase while avoiding the discontinuity of $\theta$ at $\pm \pi$. A magnitude-only view $r$ removes phase entirely; a phase-only view $(\cos\theta,\sin\theta)$ removes amplitude entirely, as shown in Figure~\ref{fig:complex_sample}.
\begin{figure}[hbt!]
    \centering
    \includegraphics[width=\linewidth]{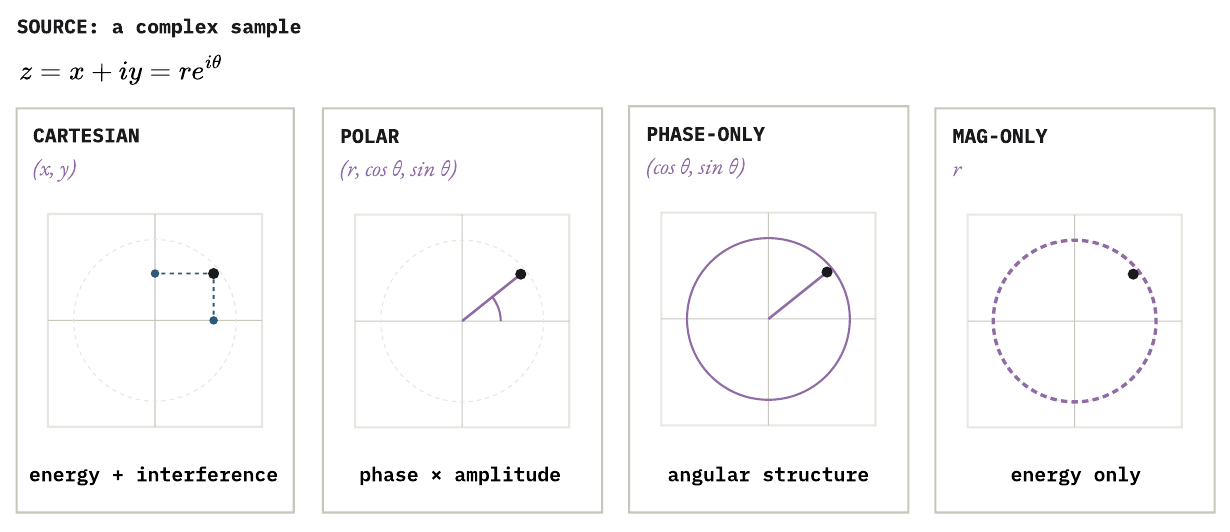}
    \caption{\textbf{Coordinate views of a complex signal.} A complex sample $z = x + iy = re^{i \theta}$ can be presented to a real-valued model through four distinct coordinate maps. Each encodes a different hypothesis about where the task-discriminative information resides.}
    \label{fig:complex_sample}
\end{figure}

This distinction matters because different complex-valued domains place label information in different parts of the data representation. In RF modulation, PSK symbols are primarily phase-structured, while QAM symbols also carry amplitude structure. In quantum wavefunctions, the magnitude $|\psi|$ can describe probability density, while phase gradients can encode momentum. In analytic-signal representations of EEG or MEG, the magnitude is an amplitude envelope and the angle is instantaneous phase (Figure~\ref{fig:data_types}). Thus, a complex-valued input does not imply that a native complex network is the right model. The relevant question is where the predictive information lives: magnitude, phase, Cartesian coordinates, or a relation between amplitude and phase.

This observation motivates the representation-first approach used throughout this paper. A CVNN is not simply compared with real-valued neural networks on stacked real and imaginary channels, but also with \textit{polar}, \textit{phase-only}, and \textit{magnitude-only} real-number baselines. These baselines are not just ablations; they test different hypotheses about the data generation and structure. If phase-only features win, the task is phase-dominant; if magnitude-only features win, the task is amplitude-dominant; if polar features win, the task likely depends on a phase-amplitude relation. After these coordinate views are controlled, if the CVNNs dominate, the remaining explanation must involve the \textit{inductive bias} or the \textit{optimization properties of complex arithmetic}.

\subsection{Complex multiplication: A constrained real linear mapping}
\label{sec:complex_layer_constraint}

The central fact about the architecture is that complex number multiplication is not an arbitrary two-channel real transformation. Multiplying an input $x+iy$ by a complex weight $a+ib$ gives
\[
    (a+ib)(x+iy)
    =
    (ax-by) + i(ay+bx).
\]
In real coordinates, this is
\[
    \begin{pmatrix}
        \Re((a+ib)(x+iy)) \\
        \Im((a+ib)(x+iy))
    \end{pmatrix}
    =
    \begin{pmatrix}
        a & -b \\
        b & a
    \end{pmatrix}
    \begin{pmatrix}
        x \\
        y
    \end{pmatrix}.
\]
A generic real-valued two-channel layer can use any matrix in $\mathbb{R}^{2\times 2}$. A complex layer is restricted to matrices of the form
\[
    aI + bJ,
    \qquad
    J =
    \begin{pmatrix}
        0 & -1 \\
        1 & 0
    \end{pmatrix}
\]
This is a two-parameter subspace of the four-parameter space of arbitrary real channel mixing matrices. The restriction is precisely what gives the complex layer its inductive bias: \textbf{it couples the real and imaginary channels as a rotation and scaling rather than allowing unconstrained channel mixing}.

This constraint has a symmetry interpretation. A global phase shift $z \mapsto e^{i\phi}z$ acts on the stacked real coordinates by the rotation matrix
\[
    R_\phi =
    \begin{pmatrix}
        \cos\phi & -\sin\phi \\
        \sin\phi & \cos\phi
    \end{pmatrix}.
\]
A two-channel real linear map $W$ is equivariant to this global phase action if it commutes with every $R_\phi$:
\[
    WR_\phi = R_\phi W
    \qquad
    \text{for all } \phi \in \mathbb{R}.
\]
The maps that satisfy this condition are exactly the complex-scalar maps.

\begin{prop}[Equivariant two-channel kernels are complex scalars]
\label{prop:complex_equivariance}
Let $W \in \mathbb{R}^{2\times 2}$. The following are equivalent:
\begin{enumerate}
    \item $WR_\phi = R_\phi W$ for all $\phi \in \mathbb{R}$;
    \item $W = aI + bJ$ for some $a,b \in \mathbb{R}$;
    \item $W$ is multiplication by the complex scalar $a+ib$ under the standard embedding of $\mathbb{C}$ into $\mathbb{R}^{2\times 2}$.
\end{enumerate}
\end{prop}
\begin{proof}
The equivalence between (2) and (3) follows from the real matrix representation of complex multiplication. If $W=aI+bJ$, then $W$ commutes with every polynomial in $J$, and therefore with $R_\phi=\exp(\phi J)$, giving (1). Conversely, if $WR_\phi=R_\phi W$ for every $\phi$, differentiating at $\phi=0$ gives $WJ=JW$. Writing
\[
    W =
    \begin{pmatrix}
        p & q \\
        r & s
    \end{pmatrix}
\]
and expanding $WJ=JW$ forces $s=p$ and $r=-q$, hence
\[
    W =
    \begin{pmatrix}
        p & q \\
        -q & p
    \end{pmatrix}
    =
    aI+bJ
\]
up to the sign convention used for $J$ and the ordering of real and imaginary channels.
\end{proof}

Proposition~\ref{prop:complex_equivariance} clarifies what a complex convolution contributes. It is not more expressive than an unconstrained stacked-real convolution. Rather, it is less expressive in a structured way: \textbf{it removes the degrees of freedom that violate global phase equivariance}. This can improve generalization when the task respects phase geometry, and it can be irrelevant or harmful when the label depends on a coordinate system that does not respect that symmetry. This places one-dimensional complex convolution inside the broader framework of group-equivariant neural networks~\cite{cohen2016group,cohen2016steerable,bronstein2021geometric,weiler2019general,worrall2017harmonic}.

The same observation also explains why a rotation-equivariant real-number baseline is not a separate substantive model in our experiments. A stacked-real convolution whose kernel taps are constrained to the $aI+bJ$ subspace follows the same forward computation and the same gradient trajectory as the corresponding complex convolution, up to channel-layout conventions. We therefore use this equivalence as a correctness witness, and reserve the empirical comparisons for native complex models versus unconstrained Cartesian, parameter-matched, and FLOP-matched real baselines.

\subsection{Backpropagation does not require holomorphic activations}
\label{sec:wirtinger}

A common misconception is that complex-valued neural networks require holomorphic nonlinearities. This would be a severe restriction: many useful nonlinearities in deep learning are bounded, piecewise-defined, or magnitude-dependent, none of which is compatible with global holomorphy. In practice, CVNNs are trained using Wirtinger calculus, which treats a complex function as a real-differentiable function of two real variables~\cite{brandwood1983complex,kreutz2009complex,hirose2006complex,trabelsi2017deep}.

Let $f:\mathbb{C}\to\mathbb{C}$ be written as
\[
    f(z) = u(x,y) + i v(x,y),
    \qquad z=x+iy.
\]
The Wirtinger derivatives are
\[
    \frac{\partial f}{\partial z}
    =
    \frac{1}{2}
    \left(
        \frac{\partial f}{\partial x}
        -
        i\frac{\partial f}{\partial y}
    \right),
    \qquad
    \frac{\partial f}{\partial \bar z}
    =
    \frac{1}{2}
    \left(
        \frac{\partial f}{\partial x}
        +
        i\frac{\partial f}{\partial y}
    \right).
\]
The Cauchy-Riemann equations are equivalent to $\partial f/\partial \bar z = 0$, so holomorphy is the special case in which the conjugate Wirtinger derivative vanishes. But ordinary backpropagation through a real-valued loss does not require this condition. For a real-valued loss $L:\mathbb{C}^n\to\mathbb{R}$ and a complex parameter $w$, gradient descent follows the steepest descent direction in the underlying real plane $(\Re w,\Im w)$, which is represented by the conjugate Wirtinger derivative. This is the convention used by modern autodiff systems for complex-valued parameters.

Consequently, activation functions such as \texttt{CReLU, ZReLU, ModReLU, Siglog}, and \texttt{ComplexCardioid} (as characterized in Figure~\ref{fig:activation-functions-graphs}) need not be holomorphic to be valid components of a trainable CVNN. They only need to be differentiable almost everywhere as real maps from $\mathbb{R}^2$ to $\mathbb{R}^2$, in the same practical sense that ReLU is differentiable almost everywhere on $\mathbb{R}$. As shown in Figure~\ref{fig:activation-characterization}, this diagnostic separates locally holomorphic but unstable activations from bounded non-holomorphic alternatives.
\begin{figure}[hbt!]
    \centering

    \begin{subfigure}[t]{0.48\textwidth}
        \centering
        \includegraphics[width=\linewidth]{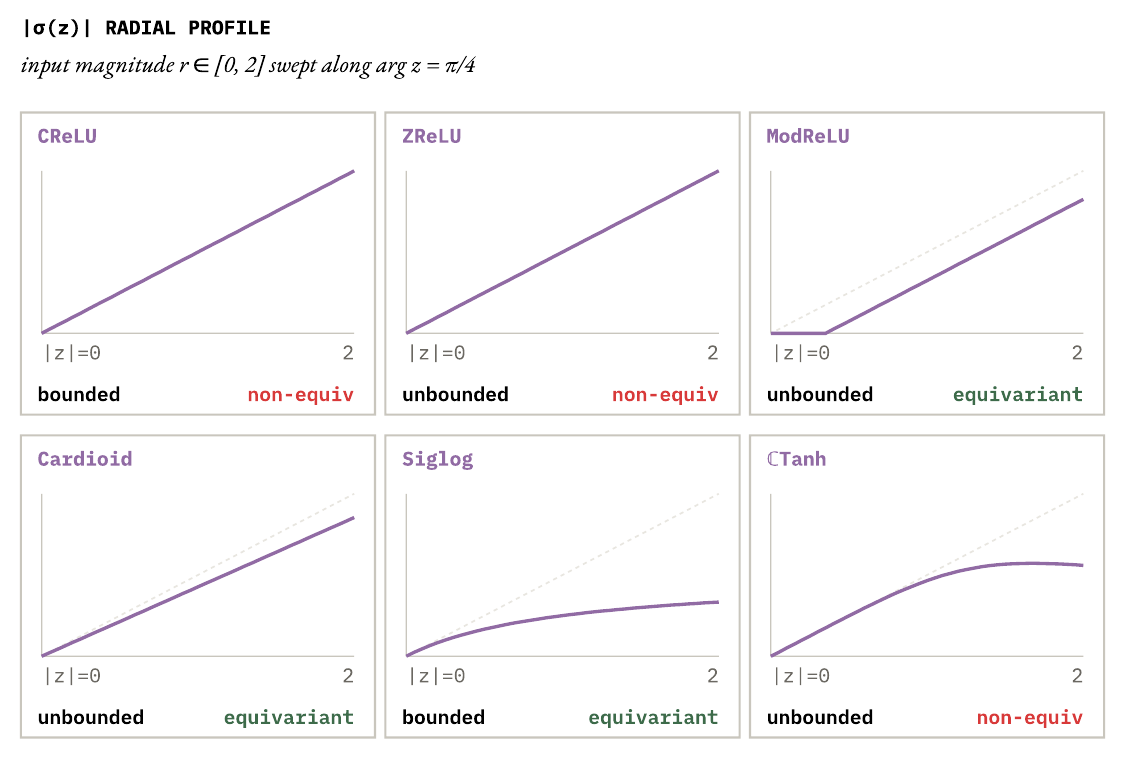}
        \caption{Activation functions used in the model.}
        \label{fig:activation-functions-plots}
    \end{subfigure}
    \hfill
    \begin{subfigure}[t]{0.48\textwidth}
        \centering
        \includegraphics[width=\linewidth]{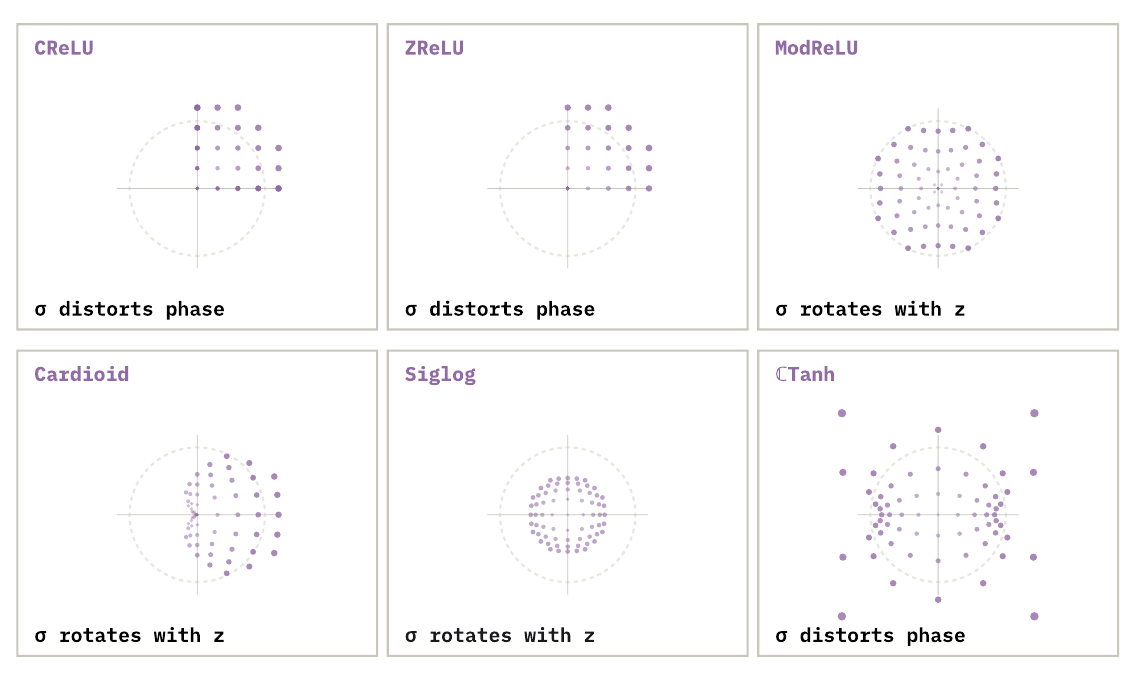}
        \caption{Functional behavior of the activation variants across input ranges.}
        \label{fig:activation-functions-graphs}
    \end{subfigure}

    \caption{Comparison of activation-function definitions and their corresponding response curves.}
    \label{fig:activation-functions}
\end{figure}
\begin{figure}
    \centering
    \includegraphics[width=0.75\linewidth]{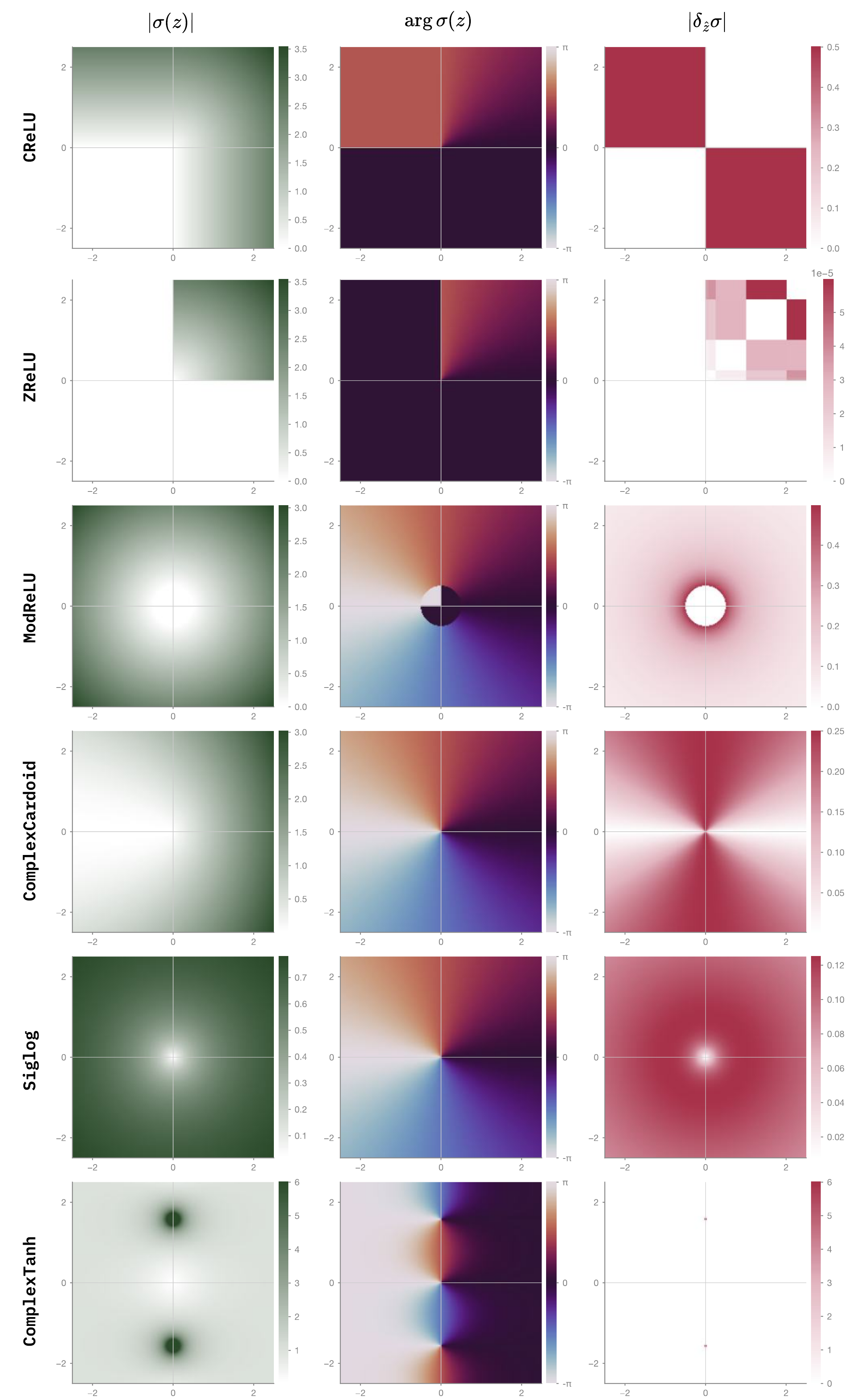}
    \caption{Each row visualizes one activation over $z \in [-3,3]^2$: magnitude $|\sigma(z)|$, phase $\arg \sigma(z)$, and the Cauchy-Riemann residual $|\delta_{\hat{z}}\sigma|$ (zero $\Leftrightarrow$ locally holomorphic). The map exposes where each activation breaks: \texttt{CReLU/ZReLU} break phase equivariance by quadrant; \texttt{ModReLU/Cardioid} concentrate their defect at the origin; \texttt{ComplexTanh} stays holomorphic but blows up at its poles.}
    \label{fig:activation-characterization}
\end{figure}
\subsection{The Liouville Activation Trilemma}
\label{sec:liouville_trilemma}

Although holomorphy is not required for backpropagation, complex nonlinearities still face a genuine design constraint. A useful activation should be nonlinear; a stable activation is often bounded or at least controlled; and a mathematically convenient complex activation would be holomorphic. Liouville's theorem shows that these desiderata cannot all hold globally.

\begin{prop}[Activation trilemma]
\label{prop:activation_trilemma}
No function $\sigma:\mathbb{C}\to\mathbb{C}$ can simultaneously be bounded on all of $\mathbb{C}$, entire, and nonconstant.
\end{prop}

\begin{proof}
By Liouville's theorem, every bounded entire function on $\mathbb{C}$ is constant.
\end{proof}

Thus, every nonconstant complex activation must give up at least one of boundedness or holomorphy. This gives a useful taxonomy of practical CVNN activations. \texttt{ComplexTanh} preserves holomorphy away from its poles but is unbounded as a meromorphic function. \texttt{Siglog} is bounded and nonconstant but explicitly non-holomorphic because it depends on $|z|$. \texttt{ModReLU, CReLU, ZReLU}, and \texttt{ComplexCardioid} also give up holomorphy, but in different ways: \texttt{ModReLU} gates magnitude, \texttt{CReLU} applies real \texttt{ReLU} separately to real and imaginary parts, \texttt{ZReLU} selects quadrants, and \texttt{ComplexCardioid} applies a phase-dependent magnitude gate~\cite{trabelsi2017deep,arjovsky2016unitary,wisdom2016full}.

For phase-bearing data, there is an additional desideratum: phase equivariance. An activation is phase-equivariant if
\[
    \sigma(e^{i\phi}z) = e^{i\phi}\sigma(z)
    \qquad
    \text{for all } z\in\mathbb{C},\ \phi\in\mathbb{R}.
\]
This property is independent of holomorphy. Some activations preserve the global phase action; others deliberately break it in exchange for sparsity or coordinate-wise gating. The key point of this paper is that the choice of activation is not a software detail. It changes both the symmetry properties and the optimization behavior of the network, and therefore must be treated as part of the experimental design.

The paper uses the Cauchy-Riemann residual as an empirical diagnostic of these activation function choices. The \textit{residual} measures the magnitude of the conjugate Wirtinger derivative, and is zero on open sets where the activation function is holomorphic. When we pair this residual with gradient norms at the network initialization, it makes the activation trade-off visible as shown in Figure~\ref{fig:activation-tradeoff}. \textit{A candidate either pays a holomorphy cost, a boundedness/stability cost, or a phase-equivariance cost}. The downstream experiments further test how these activation-level choices interact with the representation, symmetry, and hyperparameter search space.
\begin{figure}[hbt!]
    \centering
    \includegraphics[width=\linewidth]{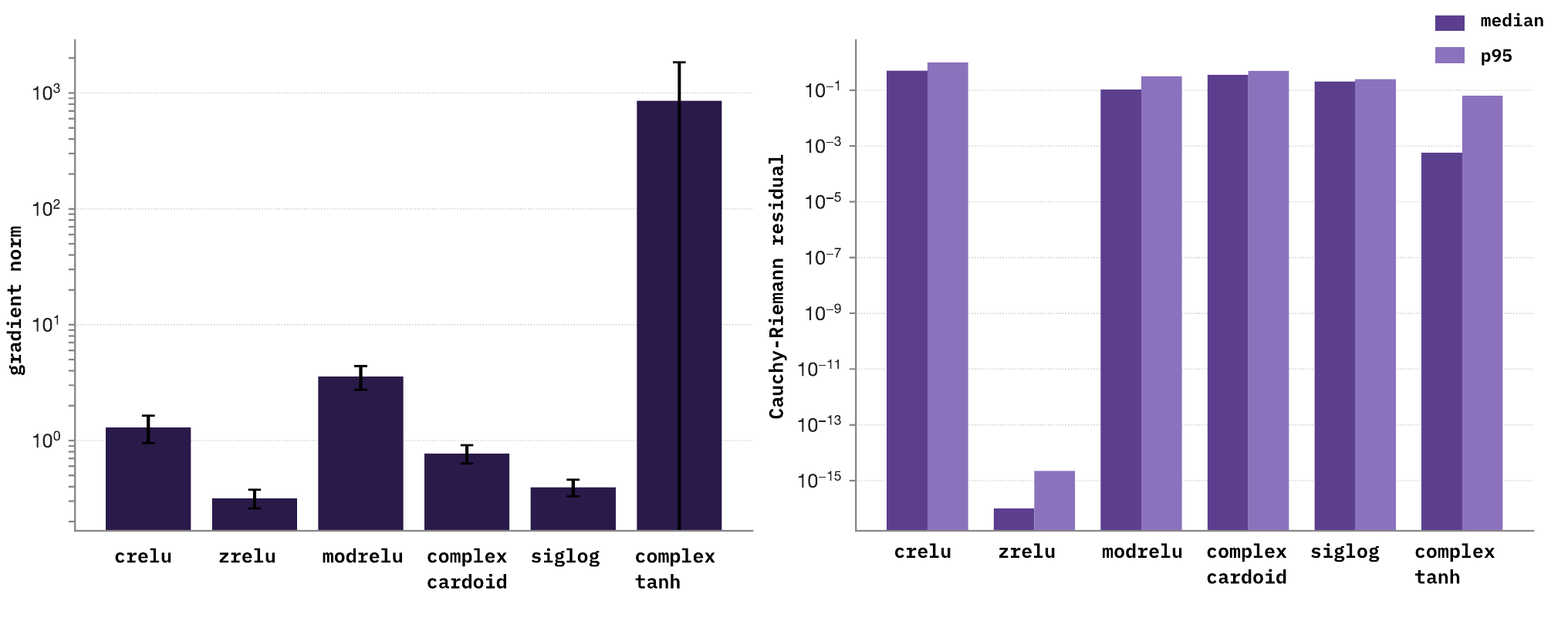}
    \caption{\textbf{Activation trade-off at initialization}. \textit{Left}: mean gradient norm at initialization for each candidate complex activation (log scale; error bars show $\pm1$ std). \textit{Right}: median and 95th-percentile Cauchy–Riemann residual evaluated on a $[-3,3]^2$ grid in $\mathbb{C}$ (log scale). No activation wins on both axes: \texttt{complex tanh} is locally holomorphic (CR residual $\approx 0$) but unbounded, its poles at $\pm i\pi/2$ lie inside the grid and inflate the gradient, while \texttt{siglog} is bounded but never holomorphic. \texttt{crelu, zrelu, modrelu}, and \texttt{complex cardioid} populate the interior of the trade-off, each sacrificing holomorphy in a different way to gain bounded, well-conditioned gradients. This is the Liouville constraint made empirical: on $\mathbb{C}$ you cannot simultaneously be bounded, non-constant, and holomorphic.}
    \label{fig:activation-tradeoff}
\end{figure}
\section{Evaluation Protocol}
\label{sec:protocol}

The goal of the evaluation is not to show that CVNNs always outperform real-valued networks. Instead, the author asks which mechanism explains a measured advantage. A complex model may win because the task is phase-bearing, because complex multiplication imposes an appropriate symmetry constraint, because a real baseline was given a poor coordinate system, or because the complex family tolerates a hyperparameter regime in which real baselines fail. We therefore separate five axes: \textbf{representation, symmetry, capacity, selection, and invariance}. Table~\ref{tab:protocol-map} summarizes the protocol. Each row is designed to remove one possible ambiguity in a CVNN-vs-real comparison.

\begin{table}[hbt!]
\centering
\small
\setlength{\tabcolsep}{5pt}
\renewcommand{\arraystretch}{1.25}
\begin{tabularx}{\linewidth}{p{0.22\linewidth} X X}
\toprule
\textbf{Diagnostic axis} & \textbf{Protocol question} & \textbf{What the comparison rules out} \\
\midrule
\textbf{Representation} &
Does the label live in magnitude, phase, Cartesian coordinates, or a phase--amplitude relation? &
A CVNN win is caused merely by giving the model a better coordinate view. \\
\midrule
\textbf{Symmetry} &
Does the $U(1)$-equivariant $aI+bJ$ constraint match the task geometry? &
Overclaiming complex arithmetic when an explicit real representation already captures the signal. \\
\midrule
\textbf{Capacity / compute} &
Does the gain survive parameter-matched and FLOP-matched real baselines? &
Attributing width, parameter count, or arithmetic budget effects to complex structure. \\
\midrule
\textbf{Selection rule} &
Does the gap survive both matched-shared-trial and independent per-family tuning? &
Reporting robustness to one shared hyperparameter trial as peak architectural superiority. \\
\midrule
\textbf{Invariance} &
Does performance survive fixed carrier-phase rotation, or only after augmentation? &
Assuming layer equivariance automatically produces end-to-end invariance. \\
\bottomrule
\\
\end{tabularx}
\caption{Diagnostic structure of the evaluation protocol. Each axis isolates a different mechanism that can explain a CVNN-vs-real gap.}
\label{tab:protocol-map}
\end{table}

\subsection{Coordinate-view baselines}
\label{sec:coordinate_views}

For every complex-valued sequence $z_t = x_t + i y_t$, we evaluate multiple real-valued views of the same underlying sample. The Cartesian view gives the model the stacked real and imaginary channels,$(x_t, y_t)$. The polar view gives explicit access to magnitude and phase,
\[
    (|z_t|, \cos\theta_t, \sin\theta_t),
    \qquad
    \theta_t = \arg z_t.
\]
We use $(\cos\theta_t,\sin\theta_t)$ rather than $\theta_t$ itself to avoid
the artificial discontinuity at the branch cut. We also include two
information-bottleneck views: a magnitude-only view $|z_t|$ and a
phase-only view $(\cos\theta_t,\sin\theta_t)$ (shown in Figure~\ref{fig:data_types}). These views encode different hypotheses about the data-generating mechanism. A phase-only model should perform well when labels depend on angle, phase lag, or phase gradient. A magnitude-only model should perform well when labels depend on energy or amplitude envelopes. A polar model should perform well when amplitude and phase interact. A native complex model should perform well when the rotation-and-scale coupling of complex multiplication matches the task geometry. Thus, a complex win is only interpretable after checking whether an explicit real coordinate view already exposes the relevant signal.

\begin{figure}[t]
  \centering
  \includegraphics[width=\linewidth]{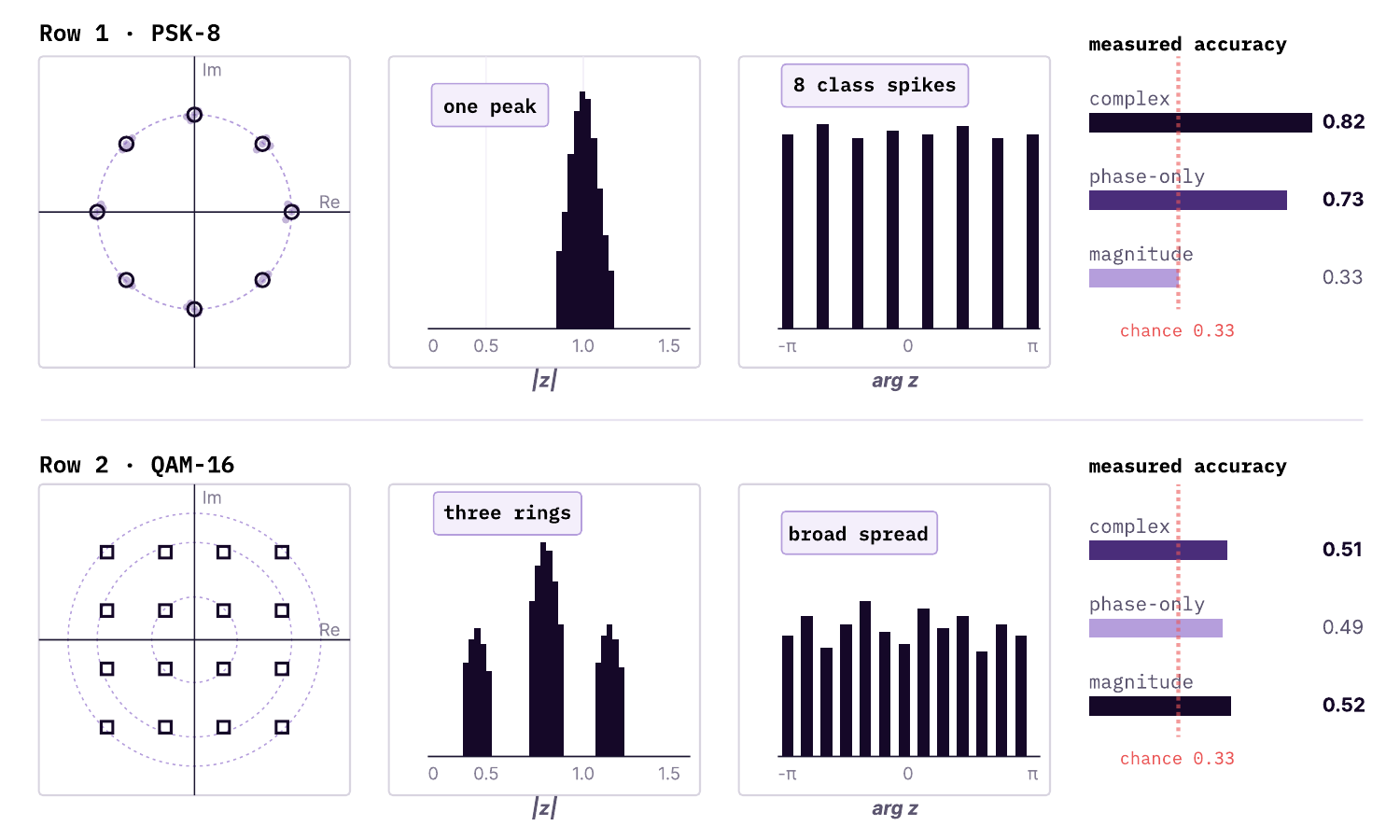}
  \caption{Two RF constellations decide which coordinate view a model
    needs. PSK-8 lives on the unit circle, so $|z|$ collapses every
    class to a single peak while $\arg z$ resolves class-specific spikes.
    QAM-16 lives on a 4$\times$4 grid, so $|z|$ resolves ring structure
    while $\arg z$ is less class-separating. The measured stress test
    accuracies in Table~\ref{tab:rf-stress} confirm the visual prediction:
    phase-only features are useful for PSK, and magnitude features are useful
    for QAM, and the complex model tracks the view that carries the class signal.}
  \label{fig:iq-anatomy}
\end{figure}
\subsection{Model-family controls}
\label{sec:model_family_controls}

Each benchmark compares the native complex model with real-valued baselines that control for different confounds, as shown in Figure~\ref{fig:model_family}.
\begin{figure}[hbt!]
    \centering
    \includegraphics[width=\linewidth]{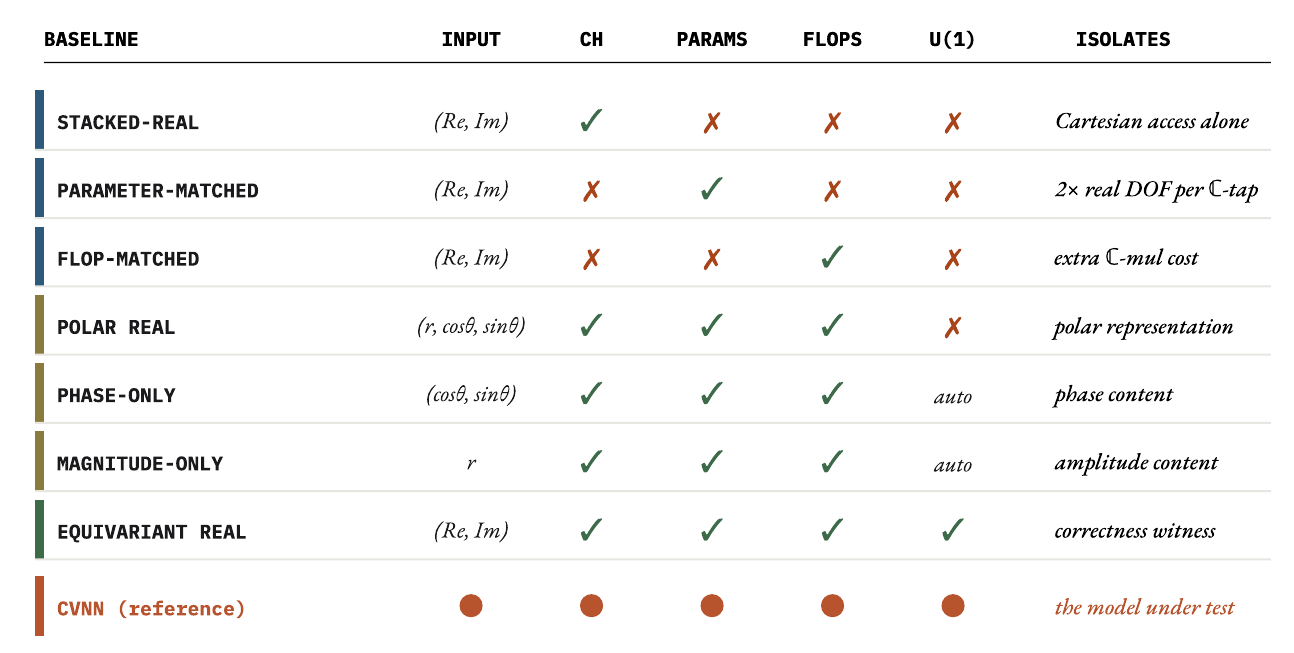}
    \caption{\textbf{Model-family controls - the real-valued baseline suite.} Each real-valued baseline holds a different set of variables constant relative to the native CVNN, isolating one confound at a time. Capacity controls (stacked, parameter-matched, FLOP-matched) test whether a complex model wins simply by carrying more wires or compute. Representation controls (polar, phase-only, magnitude) test whether the right real-valued view solves the task. The $aI + bJ$-constrained real baseline is a correctness witness for algebraic equivalence, not a separate empirical row.}
    \label{fig:model_family}
\end{figure}

\paragraph{Stacked-real baseline.}
The stacked-real baseline receives $(\Re z,\Im z)$ as ordinary real channels. This baseline tests whether native complex arithmetic helps beyond giving a real network direct access to the same Cartesian information.

\paragraph{Parameter-matched real baseline.}
The parameter-matched baseline adjusts the width of the real-valued network so that its parameter count is close to that of the complex model. This controls for the possibility that a complex model wins only because complex weights effectively carry two real degrees of freedom.

\paragraph{FLOP-matched real baseline.}
The FLOP-matched baseline adjusts real-valued capacity to approximately match the computational cost of the complex model. This is important because complex multiplication induces extra real arithmetic relative to a single real multiply-add.

\paragraph{Coordinate-specialized real baselines.}
The polar, phase-only, and magnitude-only models test whether the task can be solved by exposing the right real-valued representation. These baselines are especially important because many apparent CVNN advantages disappear once phase or magnitude is made explicit.

\paragraph{Rotation-equivariant real control.}
A stacked-real convolution constrained to the $aI+bJ$ subspace is mathematically equivalent to the corresponding complex convolution, up to channel-layout conventions. We therefore use this implementation as a correctness witness rather than as a separate empirical row. The substantive comparison is between the complex, equivalently $U(1)$-equivariant, model and unconstrained real baselines that can spend parameters outside the phase-equivariant subspace.

\subsection{Activation set}
\label{sec:activation_set}

The choice of activation function is part of the experimental design rather than an implementation detail. We characterize six complex activations: \texttt{CReLU, ZReLU, ModReLU, ComplexCardioid, Siglog}, and \texttt{ComplexTanh}. Together, they cover the main practical compromises identified by the Liouville trilemma: \textit{non-holomorphic but bounded or controlled activations, meromorphic but unbounded activations, split-channel activations, and magnitude- or phase-gated activations}.

Unless otherwise stated, the default complex activation is CReLU, following
the common deep-CVNN baseline introduced by Trabelsi et al.~\cite{trabelsi2017deep}. However, the RadioML experiments explicitly vary activation choice because the apparent CVNN-vs-real gap is not activation-independent. In particular, \texttt{CReLU, ComplexCardioid}, and \texttt{Siglog} produce large matched-shared-trial gaps under the selected high-learning-rate regime, whereas \texttt{ModReLU} and \texttt{ZReLU} yield smaller or absent gaps. This makes activation choice inseparable from the selection-rule analysis.

\subsection{Stress-test factors}
\label{sec:stress_test_factors}

The synthetic RF suite varies the data-generating mechanism rather than only the model. We include PSK-only, QAM-only, and mixed PSK+QAM tasks; low- and high-SNR regimes; unit-magnitude and unit-power normalization; fixed carrier-phase rotations; random-rotation augmentation; and activation
ablations. These factors test specific hypotheses. PSK should reward phase-aware representations. QAM should reward amplitude-sensitive representations. Fixed carrier-phase rotation should expose whether a trained model has learned, or been given, the relevant invariance. Rotation augmentation tests whether the invariance can be induced from data rather than an architectural imposition.

The quantum and EEG pilots reuse the same logic outside RF. The quantum tasks are constructed so that some labels depend on phase information that is invisible to the wave function, $|\psi|$, while others depend on broader wavefunction structure. The EEG analytic-signal tasks separate phase locking, amplitude events, and phase-amplitude coupling. These pilot studies are not intended as domain-level performance claims; they simply test whether the representation-first approach transfers to other complex-valued signal settings.

\subsection{Selection rules}
\label{sec:selection_rules_protocol}

We report two hyperparameter-selection rules because they answer different
questions, as shown in Figure~\ref{fig:selection-rules}.
\begin{figure}[hbt!]
\centering
\begin{minipage}[t]{0.48\linewidth}
\begin{tcolorbox}[
    colback=blue!5,
    colframe=blue!55,
    boxrule=0.5pt,
    arc=2mm,
    left=5pt,
    right=5pt,
    top=5pt,
    bottom=5pt,
    title=\textbf{Matched-shared-trial selection}
]
Choose the complex model's best validation trial and evaluate every real
baseline at the same trial index.

\vspace{0.4em}
\textbf{Question answered:} Which family is more robust under a shared
hyperparameter allocation?

\vspace{0.4em}
\textbf{Risk:} A large gap may reflect real-baseline instability under a
trial selected by the complex family.
\end{tcolorbox}
\end{minipage}
\hfill
\begin{minipage}[t]{0.48\linewidth}
\begin{tcolorbox}[
    colback=red!5,
    colframe=red!55,
    boxrule=0.5pt,
    arc=2mm,
    left=5pt,
    right=5pt,
    top=5pt,
    bottom=5pt,
    title=\textbf{Independent per-family selection}
]
Choose the best validation trial separately for each model family under
the same search space and budget.

\vspace{0.4em}
\textbf{Question answered:} Which family achieves the best tuned
performance?

\vspace{0.4em}
\textbf{Risk:} It may hide robustness differences that appear under a
shared hyperparameter regime.
\end{tcolorbox}
\end{minipage}
\caption{The two selection rules answer different questions. Matched-shared-trial selection measures robustness to a common trial; independent per-family selection estimates best tuned performance.}
\label{fig:selection-rules}
\end{figure}
\paragraph{Matched-shared-trial selection.}
Under matched-shared-trial selection, every family draws from the same search space and consumes the same trial-by-seed budget. The complex model selects the trial with the best validation performance, and the real baselines are evaluated at the same trial index. This rule is useful as a robustness stress test: \textit{it asks which family performs better under a shared hyperparameter allocation}.

\paragraph{Independent per-family selection.}
Under independent per-family selection, each model family chooses its own best validation trial from the same search space and budget. This rule estimates the best-achievable tuned performance for each family. It is the appropriate rule when the question is peak accuracy rather than robustness to a shared trial.

The distinction is central to the RadioML result. A large matched-shared-trial gap may indicate that the complex family tolerates a hyperparameter regime in which real baselines fail; it does not necessarily imply that the complex architecture has a large independently tuned advantage. For this reason, whenever the selection rule changes the interpretation, we report both.

\subsection{Statistical reporting}
\label{sec:statistical_reporting}

When seeds are matched across model families, we report paired differences between the complex model and each baseline. This is the design-appropriate comparison because the same seeds and trial choices induce correlated outcomes across families. For the main synthetic RF sweep with six seeds, we report per-family confidence intervals using the $t$ distribution and paired $t$ intervals for complex-minus-baseline differences. For smaller pilot studies with three seeds, we treat the results as representation maps rather than strong significance claims and avoid overstating confidence intervals.

We also report dead-seed counts when optimization instability is part of the interpretation. A dead seed is a run that collapses to near-uniform predictions and fails to recover. In the RadioML study, dead-seed counts are essential because they explain why matched-shared-trial and independent per-family selection lead to different conclusions.

\section{Representation Stress Tests}
\label{sec:results_representation}

The evaluation protocol in Section~\ref{sec:protocol} is designed to answer a mechanistic
question: \textbf{When a complex-valued model wins, is it because the task requires complex arithmetic, because the label information lives in phase, because it lives in magnitude, or because the right real-valued coordinate system was not exposed to the baseline?} We therefore begin with representation stress tests before turning to the RadioML selection-rule artifact.

The main result of this section is simple: \textit{the winning model is determined less by whether the input is formally complex-valued than by where the task-relevant information lives}. Across RF modulation, quantum wavefunctions, and EEG analytic signals, phase-dominant tasks reward phase-aware views, amplitude-dominant tasks reward magnitude views, and phase-amplitude coupling favors joint representations.

\subsection{Synthetic RF: PSK rewards phase, QAM rewards magnitude}
\label{sec:synthetic_rf_representation}

The stress tests on RF signals are spotless settings because the data-generating mechanism can be controlled. Figure~\ref{fig:iq-anatomy} shows the core intuition of these tests. \texttt{PSK} symbols lie on the unit circle, so the class information is angular: a magnitude-only representation collapses all classes to the same value, whereas phase separates the constellation points. \texttt{QAM} symbols occupy a two-dimensional grid, so amplitude carries useful ring structure, and phase alone is less decisive.

As shown in Table~\ref{tab:rf-stress}, in the \texttt{PSK}-only condition, the complex model achieves \textbf{0.821}, matching the best score, while magnitude-only remains at chance-level performance. Importantly, the phase-only and stacked-real views also learn. This result does not show that complex arithmetic is the only way to solve \texttt{PSK}, but that phase-aware representations are a necessary component of the task. In the \texttt{QAM}-only condition, the story flips. Magnitude-only is the best single view at \textbf{0.524}, while the complex model reaches \textbf{0.509}, and phase-only falls lower. For mixed \texttt{PSK+QAM}, the complex-valued model has a small edge, consistent with a task that contains both angular and amplitude structure.

\begin{table}[t]
\centering
\small
\setlength{\tabcolsep}{4pt}
\renewcommand{\arraystretch}{1.5}
\begin{tabular}{lcccccc}
\toprule
\textbf{Condition} & \textbf{Best} & \textbf{Complex} & \textbf{Real stack} &
\textbf{Phase} & \textbf{Magnitude} \\
\midrule
PSK-only & \textit{complex} & \textbf{0.821} & 0.728 & 0.735 & 0.333 \\
QAM-only & \textit{magnitude} & 0.509 & 0.502 & 0.487 & \textbf{0.524} \\
Mixed PSK+QAM & \textit{complex} & \textbf{0.507} & 0.481 & 0.487 & 0.365 \\
Low-SNR PSK & \textit{matched real} & \textbf{0.526} & \textbf{0.526} & 0.479 & 0.326 \\
High-SNR PSK & \textit{complex} & \textbf{0.949} & 0.897 & 0.892 & 0.333 \\
Unit-magnitude mixed & \textit{phase} & 0.476 & 0.490 & 0.\textbf{491} & 0.200 \\
Fixed-rotation PSK & \textit{magnitude} & 0.252 & 0.246 & 0.262 & \textbf{0.328} \\
Rotation-augmented PSK & \textit{complex} & \textbf{0.654} & 0.580 & 0.574 & 0.333 \\
\bottomrule
\\
\end{tabular}
\caption{\textbf{Synthetic RF representation stress tests}. Accuracy is mean test
accuracy across the committed stress-test seeds. The central lesson is
conditionality: PSK rewards phase-aware geometry, QAM rewards magnitude,
low SNR weakens phase-only views, and global carrier-phase shifts require
augmentation or an explicit invariance mechanism.}
\label{tab:rf-stress}
\end{table}

The rotation rows are the most important scope check. A complex convolutional layer is equivariant to global phase rotations at the linear-layer level, but the full network is not automatically invariant to an unseen carrier-phase convention. Under a fixed test-time rotation, coordinate-dependent PSK models collapse toward chance. Random rotation augmentation partially restores the complex model, but it does not turn the architecture into a guaranteed phase-invariant classifier. Thus, the RF result supports a conditional claim: \textbf{complex-valued models are useful inductive biases for phase-bearing signal learning when the representation and invariance assumptions match the task}.

\subsection{Quantum wavefunctions: phase can contain physical state information}
\label{sec:quantum_representation}

The quantum pilot tests whether the same representation logic transfers outside RF. The inputs are one-dimensional complex wavefunctions. In the momentum task, examples are Gaussian wavepackets of the form
\[
    \psi(x) = A(x)\exp(i k x + i\phi_0),
\]
and the label is the momentum class. The amplitude is randomized independently of the class, so the class information is carried by the phase gradient rather than by the wave function, $|\psi|$. Table~\ref{tab:quantum-pilot} shows that this design behaves exactly as intended. All phase-aware views solve the momentum task, while magnitude-only remains at chance. This is just the quantum analogue of the \texttt{PSK} result we needed to assert that the signal is not recoverable from amplitude alone.

\begin{table}[t]
\centering
\small
\setlength{\tabcolsep}{5pt}
\renewcommand{\arraystretch}{1.5}
\begin{tabular}{lcccccc}
\toprule
\textbf{Condition} & \textbf{Best} & \textbf{Complex} & \textbf{Real stack} &
\textbf{Phase} & \textbf{Polar} & \textbf{Magnitude} \\
\midrule
Momentum & \textit{phase / complex} & \textbf{1.000} &\textbf{1.000} & \textbf{1.000} & \textbf{1.000} & 0.250 \\
Potential inverse & \textit{real stack} & 0.431 & \textbf{0.462} & 0.405 & 0.418 & 0.331 \\
Global phase shift & \textit{polar} & 0.323 & 0.382 & 0.374 & \textbf{0.428} & 0.331 \\
Global phase augmented & \textit{polar} & 0.395 & 0.392 & 0.362 & \textbf{0.403} & 0.331 \\
\bottomrule
\\
\end{tabular}
\caption{Quantum wavefunction pilot, standard preset with three seeds,
128 examples per class, 96 grid points, and 140 training steps. Chance is
0.25 for momentum classification and 0.20 for potential classification.
The result supports the representation thesis in a physics setting, but not
a broad complex-dominance claim.}
\label{tab:quantum-pilot}
\end{table}

The wavefunction potential-inverse task gives the necessary counterweight to this study. The model must predict one of five potential families after a short split-step evolution, and the best representation is not the native complex-valued model but the stacked-real baseline. The global-phase rows make the same limitation visible again: \textit{a global phase shift is physically irrelevant, but the tested complex stack does not automatically remove it}. A randomized global-phase augmentation partially improves the complex-valued model; still, the polar view remains the strongest. Thus, the result is not \textit{complex models solve quantum states}; it is much narrower and more useful: \textbf{when the physical variable is encoded in phase, magnitude-only representations fail; when the task mixes other structure, explicit real coordinate views can match or surpass the native complex layers}.

\subsection{EEG analytic signals: phase, amplitude, and coupling select different views}
\label{sec:eeg_representation}

The EEG pilot study tests a representation pattern that appears in real biomedical signal processing (the synthetic design is by choice). A real sensor trace can be mapped through a Hilbert or wavelet transform into a complex analytic signal (the magnitude is an amplitude package, and the angle is an instantaneous phase). This makes EEG a natural test of whether the protocol can distinguish phase-dominant, amplitude-dominant, and phase-amplitude coupled tasks.

Table~\ref{tab:eeg-pilot} shows that the pilot design separates these cases without a spot. For phase locking, the label is the phase lag between channels; phase-only features solve the task, and magnitude-only remains at chance. For amplitude events, the label is the channel containing an amplitude burst; magnitude-only is the best view. For phase-amplitude coupling, neither phase-only nor magnitude-only is enough, while the polar view reaches the best accuracy. The two reference-phase rows further show that explicit phase features can remain strong under the tested phase-shift conditions, while the native complex model is competitive but not dominant.

\begin{table}[t]
\centering
\small
\setlength{\tabcolsep}{5pt}
\renewcommand{\arraystretch}{1.5}
\begin{tabular}{lcccccc}
\toprule
\textbf{Condition} & \textbf{Best view} & \textbf{Complex} & \textbf{Real stack} &
\textbf{Phase} & \textbf{Polar} & \textbf{Magnitude} \\
\midrule
Phase locking & \textit{phase} & 0.926 & 0.981 & \textbf{1.000} & 0.971 & 0.250 \\
Amplitude event & \textit{magnitude} & 0.494 & 0.311 & 0.256 & 0.933 & \textbf{0.946} \\
Phase--amplitude coupling & \textit{polar} & 0.308 & 0.343 & 0.240 & \textbf{0.881} & 0.263 \\
Reference phase shift & \textit{phase} & 0.920 & 0.971 & \textbf{0.994} & 0.971 & 0.250 \\
Reference phase augmented & \textit{phase} & 0.891 & 0.984 & \textbf{0.990} & 0.942 & 0.250 \\
\bottomrule
\\
\end{tabular}
\caption{EEG analytic-signal pilot, standard preset with three seeds,
128 examples per class, four channels, 64 time steps, 120 training steps,
and CReLU. Chance is 0.25. The result is a representation map, not a claim
about clinical EEG performance.}
\label{tab:eeg-pilot}
\end{table}

This study strengthens the paper's central framing because it is not an RF task, first of all. Secondly, the same complex-valued input format supports three different answers: \textit{phase-only for phase locking, magnitude-only for amplitude bursts, and polar for phase-amplitude coupling}. A native complex-valued model is not automatically the best representation, simply because the signal is complex-valued.

Across all three domains (RF, quantum wavefunctions, and EEG analytic-signals), the representation stress tests rule out the \textit{complex beats real} interpretation. Complex-valued models help most when the task geometry aligns with the phase-and-amplitude coupling built into complex multiplication. But explicit real coordinate systems can match or outperform them when they expose the relevant variable directly. The appropriate conclusion is therefore conditional: \textbf{complex arithmetic is a useful inductive bias, not a universal substitute for testing where the information actually lives}.
\section{RadioML: When a Large CVNN Gap Becomes a Selection Artifact}
\label{sec:radioml}

The representation stress tests in Section~\ref{sec:results_representation} show when complex-valued models should help: when the label depends on phase, amplitude, or their coupling in a way that aligns with the model's coordinate and symmetry assumptions. We now turn to a different question: \textit{when a large CVNN-vs-real gap appears on real RF data, how much of that gap is architectural, and how much is induced by the hyperparameter-selection rule?}

We study this question on a controlled 3-class subset of RadioML 2018.01A~\cite{8267032}: BPSK, QPSK, and 8PSK at SNR levels $\{-10,-6,-2,2,6,10,14,18\}\ \mathrm{dB}$,
with \texttt{max\_per\_class\_per\_snr=256} and sequence length 128. This subset is intentionally narrower than the full 24-class archive. It keeps the task phase-dominant enough to compare against the synthetic PSK stress tests, while remaining large enough to expose real optimization dynamics. The result is a useful case study: \textbf{a headline complex-valued advantage that looks dramatic under one selection rule becomes much smaller under another}.

\subsection{A large matched-shared-trial gap}
\label{sec:radioml_initial_gap}

Under the \texttt{CReLU} configuration, matched-shared-trial selection produces the result in Table~\ref{tab:radioml-crelu}. The complex model reaches 0.7293 test accuracy, while the best real baseline reaches 0.4999, giving a 22.94 PP gap. Read in isolation, this looks like strong evidence that the complex architecture is substantially better.

\begin{table}[t]
\centering
\small
\setlength{\tabcolsep}{7pt}
\renewcommand{\arraystretch}{1.15}
\begin{tabular}{lccc}
\toprule
\textbf{Family} & \textbf{Test accuracy} & \textbf{Std.} & \textbf{Parameters} \\
\midrule
Complex (CReLU) & \textbf{0.7293} & 0.0085 & 58,886 \\
Real (stacked) & 0.4583 & 0.1394 & 29,891 \\
Real ($\approx$ params) & 0.4999 & 0.1327 & 58,413 \\
Real ($\approx$ FLOPs) & 0.4245 & 0.1415 & 117,123 \\
\bottomrule
\\
\end{tabular}
\caption{RadioML 2018.01A subset under CReLU, using matched-shared-trial selection after a 16-trial by 6-seed sweep. The apparent CVNN-vs-real gap is large, but later ablations show that most of it is induced by real-baseline instability under the matched trial.}
\label{tab:radioml-crelu}
\end{table}

The large standard deviations of the real baselines are the first warning sign. The complex model is stable across seeds, while the real baselines split into successful and failed runs. This suggests that the matched trial selected by the complex family may place the real baselines in an unstable optimization regime.

\subsection{Activation ablation reveals the instability}
\label{sec:radioml_activation_ablation}

To test whether the \texttt{CReLU} gap reflects a general complex-valued advantage, we rerun the same matched-shared-trial protocol across five complex activations. The real baselines continue to use \texttt{ReLU}. Table~\ref{tab:radioml-activation-ablation} summarizes the result, and visualized in Figure~\ref{fig:radioml_activation}.

\begin{table}[t]
\centering
\small
\setlength{\tabcolsep}{5pt}
\renewcommand{\arraystretch}{1.15}
\begin{tabular}{lcccc}
\toprule
\textbf{Activation} & \textbf{Complex} & \textbf{Best real} & \textbf{Gap (pp)} & \textbf{Dead real seeds} \\
\midrule
CReLU & 0.7293 & 0.4999 & +22.94 & 3/9 \\
ComplexCardioid & 0.7282 & 0.5028 & +22.54 & 3/9 \\
Siglog & 0.7014 & 0.4689 & +23.25 & 3/9 \\
ModReLU & 0.6683 & 0.6940 & $-2.58$ & 0/9 \\
ZReLU & 0.7330 & 0.7039 & +2.91 & 0/9 \\
\bottomrule
\\
\end{tabular}
\caption{\textbf{Activation ablation on the RadioML subset under matched-shared-trial selection}. The large 22--23 pp gaps appear exactly in the rows with dead real-baseline seeds. Under ModReLU and ZReLU, where the real baselines remain stable, the gap is small or reversed.}
\label{tab:radioml-activation-ablation}
\end{table}
\begin{figure}[hbt!]
    \centering
    \includegraphics[width=\linewidth]{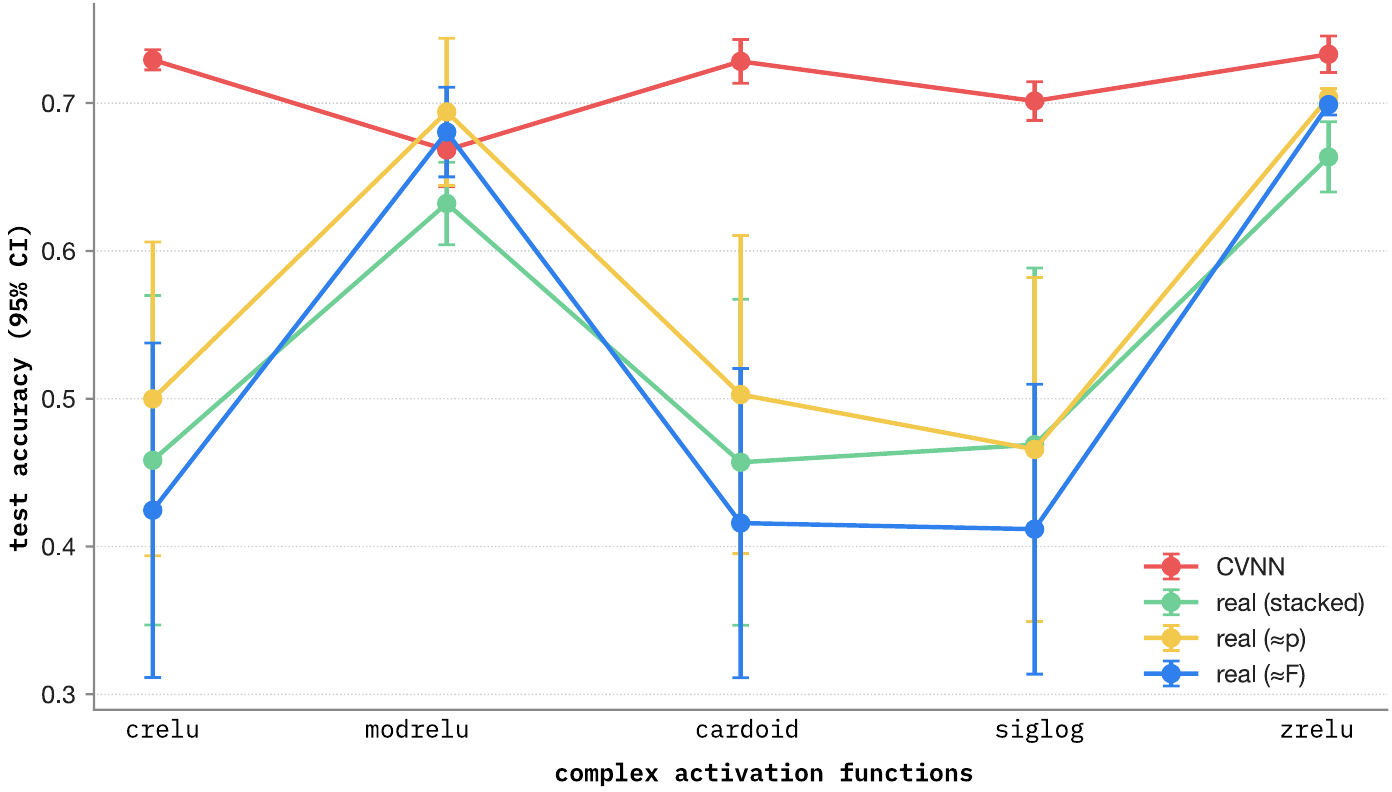}
    \caption{\textbf{RadioML activation ablation}. The apparent CVNN advantage is a real-baseline collapse, not a complex-side gain. Selected test accuracy (mean $\pm$ 95\% CI over seeds, matched-shared-trial selection) versus the complex-side activation, with one line per model family. Real baselines, which use ReLU regardless of the complex activation, nevertheless swing from $\sim0.45$ to $\sim0.70$ depending on which complex activation drives hyper-parameter selection. The CVNN line stays nearly flat (0.668–0.733). The 23 PP gap reported in [Table~\ref{tab:radioml-activation-ablation}] under crelu is therefore an artifact of the shared-trial protocol selecting a high-LR / wide-hidden configuration that the real baseline cannot tolerate, not evidence that complex representations are intrinsically stronger on RadioML.}
    \label{fig:radioml_activation}
\end{figure}
The pattern changes the interpretation. Complex accuracy is relatively stable across activations, but real-baseline accuracy swings by more than 20 points. The large \texttt{CReLU, ComplexCardioid}, and \texttt{Siglog} gaps are therefore not primarily caused by the complex model pulling far ahead. They are caused by the real baselines failing under the matched trial. By contrast, \texttt{ModReLU} and \texttt{ZReLU} produce stable real baselines and yield gaps of only a few points, consistent with the smaller complex-vs-real advantages usually reported in the RadioML literature.

\subsection{Mechanism: a first-step explosion into a dead region}
\label{sec:radioml_mechanism}

The activation ablation identifies an instability but does not explain it. We therefore instrument the selected configurations with per-step telemetry: training loss, total parameter-gradient norm, per-parameter gradient norm, and maximum parameter magnitude. The telemetry covers five activations, four model families, and three seeds, for 60 instrumented trajectories capped at 200 steps, as shown in Figure~\ref{fig:radioml_telemetry}.

\begin{figure}
    \centering
    \includegraphics[width=\linewidth]{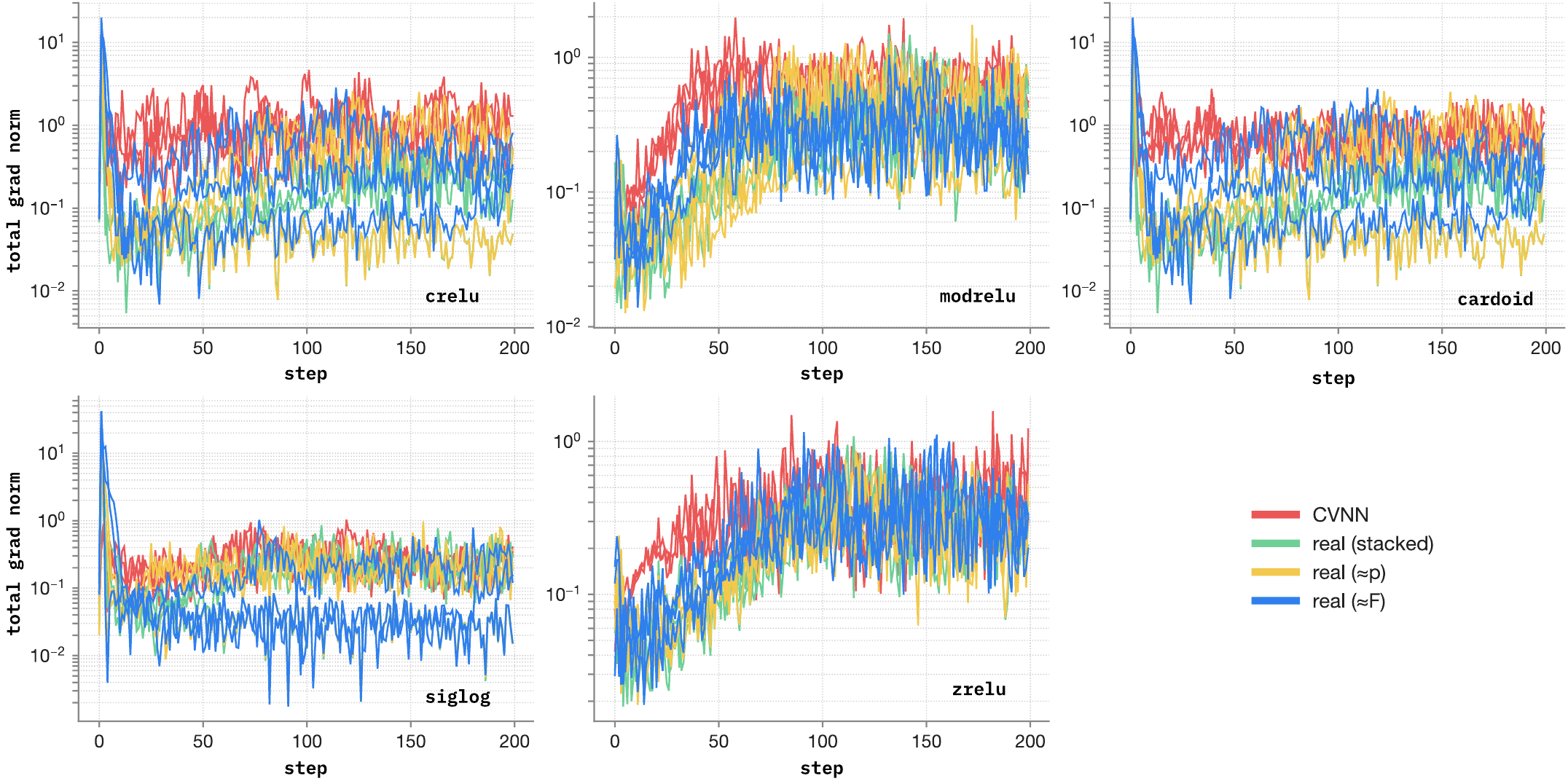}
    \caption{\textbf{Per-step total gradient norm (log scale), one panel per complex activation, 3 seeds × 4 families per panel; first 200 steps shown}. Curves are colored by model family ($\mathbb{C}$VNN vs. matched-parameter / stacked / wide real baselines). Under \texttt{crelu, cardioid}, and \texttt{siglog}, real baselines exhibit a step-1 spike of $10^{1}$–$10^{2}$ in total grad norm before relaxing, the explosion-into-dead-region failure that drives the inflated accuracy gaps in [Table tab:ablation-summary]. Under \texttt{modrelu} and \texttt{zrelu}, no spike appears and the family lines remain co-located throughout training; correspondingly the cross-family accuracy gap collapses to $\pm3$ PP.}
    \label{fig:radioml_telemetry}
\end{figure}

The failure occurs almost immediately. For \texttt{CReLU, ComplexCardioid}, and \texttt{Siglog}, matched-shared-trial selection chooses high-learning-rate, wide-hidden configurations that the complex model tolerates. At the first AdamW step, the real baselines receive a much larger update in the classifier head than the complex model. In the unstable rows, the real models show step-1 losses between roughly 3 and 36, total gradient norms between roughly 6 and 42, and \texttt{head.weight} gradient norms between roughly 6 and 41. The complex model, on the same data and selected trial, remains near ordinary cross-entropy scale, with step-1 losses around 1.07--1.28 and much smaller head gradients.

The dead runs then collapse to uniform prediction. For this 3-class task, uniform prediction gives
\[
    \log 3 \approx 1.099,
\]
and the failed real seeds converge to this loss with near-zero gradient norm and no test-accuracy improvement. In other words, the matched-shared-trial rule selects a hyperparameter regime in which the complex model trains, but some real seeds are pushed into a dead-ReLU region from which they do not recover.

\begin{table}[t]
\centering
\small
\setlength{\tabcolsep}{7pt}
\renewcommand{\arraystretch}{1.15}
\begin{tabular}{lccc}
\toprule
\textbf{Activation} & \textbf{Learning rate} & \textbf{Dead seeds} & \textbf{Step-1 head grad. max} \\
\midrule
CReLU & 0.024 & 3/9 & 13.7 \\
ComplexCardioid & 0.024 & 3/9 & 19.5 \\
Siglog & 0.040 & 3/9 & 40.4 \\
ModReLU & 0.008 & 0/9 & 0.17 \\
ZReLU & 0.0024 & 0/9 & 0.13 \\
\bottomrule
\\
\end{tabular}
\caption{Dead-seed counts for the real baselines under the matched-shared trial selected by each activation. The instability appears when the selected learning rate crosses roughly $2\times 10^{-2}$, producing a large first-step classifier-head gradient.}
\label{tab:dead-seeds}
\end{table}

The key point is not that \texttt{CReLU} is intrinsically unusable. The key point is that the matched-shared-trial rule can select a trial in which the complex family survives, but the real baselines do not. The measured gap then becomes a mixture of architectural performance and optimization robustness.

\subsection{The threshold replicates outside RadioML}
\label{sec:radioml_synthetic_replication}

To check whether the mechanism is specific to RadioML artifacts such as pulse shaping, carrier offset, or channel effects, we rerun the telemetry harness on synthetic AWGN-only IQ data using the same matched-trial hyperparameters. Table~\ref{tab:synthetic-replication} shows that the dead-seed pattern replicates almost exactly.

\begin{table}[t]
\centering
\small
\setlength{\tabcolsep}{5pt}
\renewcommand{\arraystretch}{1.15}
\begin{tabular}{lcccc}
\toprule
\textbf{Activation} & \textbf{RadioML dead/9} & \textbf{Synthetic dead/9} & \textbf{RadioML max grad.} & \textbf{Synthetic max grad.} \\
\midrule
CReLU & 3 & 3 & 19.9 & 34.1 \\
ComplexCardioid & 3 & 3 & 19.9 & 34.1 \\
Siglog & 3 & 2 & 42.1 & 55.7 \\
ModReLU & 0 & 0 & 0.3 & 0.5 \\
ZReLU & 0 & 0 & 0.2 & 0.7 \\
\bottomrule
\\
\end{tabular}
\caption{Synthetic-data replication of the RadioML instability. Applying the same matched-shared-trial hyperparameters to AWGN-only synthetic IQ data reproduces the same dead-seed pattern, indicating that the mechanism is hyperparameter-driven rather than RadioML-specific.}
\label{tab:synthetic-replication}
\end{table}

This replication narrows the explanation. The instability is not caused by a peculiarity of the RadioML loader, channel model, or SNR subset. It is a general interaction between the matched trial, the selected learning rate, and the readout dynamics of the real baselines.

\subsection{Learning rate, not activation identity, is the dominant axis}
\label{sec:radioml_lr_factorial}

The dead-seed table might suggest that some activations are intrinsically unstable while others are intrinsically stable. But the unstable rows also select much larger learning rates. To disambiguate activation identity from learning rate, we run a $2\times2$ factorial on synthetic AWGN data:
\[
    \{\mathrm{CReLU}, \mathrm{ZReLU}\}
    \times
    \{0.0236, 0.0024\},
\]
holding all other hyperparameters fixed at the corresponding matched-trial configuration.

\begin{table}[t]
\centering
\small
\setlength{\tabcolsep}{7pt}
\renewcommand{\arraystretch}{1.15}
\begin{tabular}{llcc}
\toprule
\textbf{Cell} & \textbf{Activation / LR} & \textbf{Dead seeds} & \textbf{Step-1 head grad. max} \\
\midrule
Matched high-LR & CReLU / 0.0236 & 4/9 & 31.4 \\
Swap low-LR onto CReLU & CReLU / 0.0024 & 0/9 & 2.3 \\
Matched low-LR & ZReLU / 0.0024 & 0/9 & 0.3 \\
Swap high-LR onto ZReLU & ZReLU / 0.0236 & 1/9 & 6.3 \\
\bottomrule
\\
\end{tabular}
\caption{Learning-rate/activation factorial on synthetic AWGN data. The failure follows learning rate more strongly than activation identity: lowering CReLU to the ZReLU learning rate removes dead seeds, while raising ZReLU to the CReLU learning rate introduces the first dead seed.}
\label{tab:lr-factorial}
\end{table}

The factorial makes the mechanism clear. Learning rate is the dominant axis; activation is a secondary modulator. \texttt{CReLU} is more sensitive at the same high learning rate, but it does not fail at the safer low learning rate. Conversely, \texttt{ZReLU} is stable at its selected low learning rate, but begins to fail when forced into the high-learning-rate regime. The apparent activation story is therefore mostly a selection story.

\subsection{Independent selection collapses the headline gap}
\label{sec:radioml_selection}

The practical consequence is shown in Table~\ref{tab:selection-collapse}. On the same RadioML subset, with the same 16-trial search box, the \texttt{CReLU} complex gap drops from 22.94 PP under matched-shared-trial selection to 2.46 PP under independent per-family tuning.

\begin{table}[t]
\centering
\small
\setlength{\tabcolsep}{8pt}
\renewcommand{\arraystretch}{1.15}
\begin{tabular}{lccc}
\toprule
\textbf{Selection rule} & \textbf{Complex} & \textbf{Best real} & \textbf{Gap} \\
\midrule
Matched-shared-trial & 0.7293 & 0.4999 & +22.94 PP \\
Independent per-family & 0.7293 & 0.7047 & +2.46 PP \\
\bottomrule
\\
\end{tabular}
\caption{Selection-rule contrast on the RadioML CReLU configuration. The same data and search space produce a large gap under matched-shared-trial selection but a small gap under independent per-family selection.}
\label{tab:selection-collapse}
\end{table}

This does not mean the matched-shared-trial result is invalid. It answers a different question: \textbf{which family is more robust when forced to share the complex model's selected hyperparameter trial?} Under that question, the complex model genuinely performs better. But it should not be reported as a peak-performance architectural advantage. The independent-selection result shows that once each family avoids its unstable regimes, the complex advantage is much smaller.

The RadioML result is therefore best read as a methodological finding. A large CVNN-vs-real gap can be manufactured by the interaction of matched-shared-trial selection, high learning rate, and real-baseline readout instability. The complex model's robustness is real, but it is not the same as a large, independently tuned accuracy advantage.

For future CVNN benchmarking, we recommend reporting both selection rules. Matched-shared-trial selection should be framed as a robustness stress test; independent per-family selection should be used as the peak-performance estimate. When the two disagree, dead-seed counts and per-step gradient telemetry should be reported before interpreting the gap as architectural.
\section{Discussion}
\label{sec:discussion}

The experiments support a narrower and more useful claim than a generic "\textit{complex-valued networks outperform real-valued networks}" conclusion. The complex-valued models are best understood as structured inductive biases for data with phase and amplitude. They thrive when the task geometry aligns with complex number multiplication operation; they can be matched or surpassed by explicit real-valued coordinate systems, and additionally, their advantage can be amplified by hyperparameter-selection rules.

\subsection{What complex-valued models buy}
\label{sec:what_complex_buys}

CVNNs first support representation alignment. Many classes of signals decompose naturally into amplitude and phase. Quantum wavefunctions, RF IQ samples, analytic biomedical signals, and Fourier domain measurements all possess structures that are awkward to describe using only a scalar real-valued channel. These quantities are coupled in a single algebraic structure via a complex representation. 

The second advantage is the symmetry-constrained parameterization. In Proposition~\ref{prop:complex_equivariance}, we show that complex multiplication is equivalent to a two-channel real map restricted to the subspace $aI+bJ$. This restriction is not an additional expressivity, but a structured reduction of expressivity. A generic stacked-real layer can learn both phase-equivariant and phase-non-equivariant interactions between channels. A complicated layer only retains the phase-equivariant component. When the task respects this geometry, the constraint can improve generalization by removing irrelevant degrees of freedom.

The third benefit is some shared hyperparameter regimes under optimization robustness. The RadioML study in Section~\ref{sec:radioml} shows that the complex
family tolerates high-learning-rate trials that push some real baselines into
dead regions. This is a real benefit under matched-shared-trial selection, but it is not the same as a large independently tuned accuracy benefit. It should thus be reported as a robustness to a shared trial, not as an unqualified superiority of architecture.

\subsection{What complex-valued models do not buy}
\label{sec:what_complex_does_not_buy}

The experiments also suggest a few things that complex-valued models do not automatically provide

First, complex-valued inputs do not mean complex-valued networks are the best option. In the QAM-only stress test, the strongest single-coordinate view is magnitude. Magnitude-only features outperform the native complex model in the EEG amplitude-event task. The stacked-real baseline outperforms the complex model in the quantum potential inverse task. These aren’t protocol failures; they are the point of the protocol. They indicate that the relevant variable could be magnitude, phase, Cartesian structure, or phase-amplitude coupling, depending on the task.

Second, complex layers do not provide end-to-end phase invariance. The linear complex layer is equivariant to global phase rotations in the bias-less case, but the full network involves biases, nonlinearities, pooling, readout choices, and training data. The fixed-rotation RF and quantum rows show that performance can degrade due to unobserved global phase shifts. If invariance to carrier phase or global phase is desired, it must be either tested directly or built in through architecture, augmentation, or the output representation.

Third, complex arithmetic does not obviate the need for fair real-valued baselines. The complex input sees the same information as a real model with Cartesian coordinates. A polar model might perhaps reveal the relevant variables more directly. A parameter-matched or FLOP-matched model controls for capacity and compute. Without these comparisons, a CVNN win is ambiguous; it could be a matter of representation, symmetry, number of parameters, optimization, or selection.

\subsection{Selection rules answer different scientific questions}
\label{sec:selection_rules_discussion}

The RadioML result (Section~\ref{sec:radioml}) explains why the choice rule has to be included in the claim. Both matched-shared-trial selection and independent per-family selection are valid, but they estimate different quantities. Matched-shared-trial selection asks if a family is robust with a shared hyperparameter assignment. The complex model can win with this rule because it survives a trial that destabilizes real baselines. This is useful when the experimental question is robustness under limited tuning. Independent per-family selection asks which family performs best tuned under the same search budget. This rule enables each family to avoid failure-inducing hyper-parameter regimes. This is the more appropriate estimate for the peak performance of the model family.

The difference between the two rules is in itself diagnostic. If the CVNN advantage is large in the case of matched-shared-trial selection but small in the case of independent selection, this should be interpreted as a robustness or optimization effect. If the advantage survives both rules, then the evidence for an architectural or representational advantage is stronger. This distinction is particularly important in CVNN benchmarking, where matched trials are often used to ensure apparent fairness but can inadvertently select a regime that is advantageous for one family and unstable for another.

\subsection{Practical recommendations for CVNN evaluation}
\label{sec:practical_recommendations}

The results suggest the following practice of evaluation for future comparisons of CVNNs.

First, report coordinate-view baselines.  Wherever the domain allows them, a native complex model should be compared to Cartesian real, polar, phase-only, and magnitude-only views. These baselines classify the task as phase dominant, amplitude dominant, or phase-amplitude coupled.

Second, report capacity and compute controls.  To compare with a complex layer, the real baseline needs to be matched in parameters and FLOPs because a complex layer influences the number of real degrees of freedom and the FLOP count per operation.

Third, whenever you can, give both selection rules. Matched-shared-trial selection should be seen as a robustness stress test. The peak-tuned comparison should be reported separately for each family selection. If there are big differences between the two, consider them a result, not a nuisance.

Fourth, when gaps are large, report instability diagnostics. Dead seed counts, per-step gradient norms, and early training losses can distinguish a genuine architectural advantage from a baseline collapse. For the RadioML experiment, these diagnostics are vital: without them, the 22.94 percentage-point matched-trial gap would be easily misread as a large architectural win.

Test invariance finally directly. If robustness to global phase, carrier offset, or phase convention is a claim, the benchmark should include fixed-rotation and augmentation stress tests. Layer-wise equivariance is not sufficient for end-to-end invariance.

\subsection{Scope of the claim}
\label{sec:scope_of_claim}

The strongest claim supported by this paper is not that complex-valued networks are universally better than real-valued networks. The supported claim is conditional:

\begin{quote}
Complex-valued networks help when their representation, symmetry constraint, and optimization behavior align with the task. Their advantage is largest when phase and amplitude are both meaningful and when the $U(1)$-equivariant constraint removes irrelevant real-valued degrees of freedom. Their advantage shrinks or disappears when explicit real coordinate views expose the relevant signal, or when independently tuned real baselines avoid unstable hyperparameter regimes.
\end{quote}

This framing makes the results falsifiable. A new domain should not be classified as \textit{complex} and assumed to favor CVNNs. Instead, the evaluation should ask where the label information lives, which symmetry should be preserved, whether the desired invariance is actually achieved, and whether the comparison survives both matched and independently tuned selection rules.

\section{Limitations}
\label{sec:limitations}

This study is intentionally mechanism-focused, and its scope is narrower than a full benchmark claim.

\paragraph{RadioML subset rather than the full archive.}
The real-data study uses a controlled 3-class subset of RadioML 2018.01A: BPSK, QPSK, and 8PSK across eight SNR levels. This subset is useful because it is phase-dominant and exposes the selection-rule artifact clearly. However, it does not support claims about the full 24-class RadioML archive, especially amplitude-heavy classes such as QAM. A full-scale run over all modulation types and SNRs is necessary before making broad RF application claims.

\paragraph{Synthetic QAM is tested, but real QAM is not.}
The synthetic RF stress tests include QAM-only and mixed PSK+QAM settings, where magnitude and phase--amplitude structure matter. The RadioML mechanism study, however, remains PSK-focused. This means the paper can claim that QAM-like settings require different coordinate views, but it cannot yet claim that the same behavior holds on real RadioML QAM classes.

\paragraph{No guaranteed end-to-end phase invariance.}
The analysis shows that complex convolution imposes a $U(1)$-equivariant linear constraint, but the full network is not automatically invariant to global phase shifts. Biases, nonlinearities, pooling, readout choices, and training distribution all affect end-to-end invariance. The fixed-rotation stress tests show that unseen carrier-phase or global-phase shifts can still degrade performance. Future work should test architectures that impose invariance more directly, rather than relying only on complex arithmetic.

\paragraph{Application pilots are small by design.}
The quantum and EEG experiments are representation probes, not domain benchmarks. The quantum pilot uses one-dimensional synthetic wavefunctions and a limited set of potential families. The EEG pilot uses synthetic analytic signals rather than clinical or large-scale neurophysiological recordings. These pilots support the transfer of the representation-first protocol beyond RF, but they do not establish state-of-the-art results in physics or biomedical machine learning.

\paragraph{Activation and architecture coverage is limited.}
The experiments study a compact one-dimensional convolutional architecture and a finite set of common complex activations. They do not cover transformer-style architectures, recurrent complex models, two-dimensional complex convolutions, MRI-scale reconstruction models, or learned invariance mechanisms. The conclusions should therefore be read as evidence about the mechanisms tested here rather than as an exhaustive statement about all CVNN designs.

\paragraph{The optimization artifact is diagnostic, not universal.}
The RadioML selection-rule artifact is real in our setting, but we do not claim that all matched-shared-trial CVNN results are inflated. The mechanism depends on the interaction between the selected learning rate, activation behavior, readout gradients, and real-baseline instability. The broader recommendation is methodological: when a large CVNN-vs-real gap appears under a shared-trial protocol, it should be checked against independent per-family tuning and early-step optimization diagnostics.
\section{Conclusion}
\label{sec:conclusion}

Since many signals are inherently complex-valued, CVNNs are frequently justified. This research demonstrates that where the task-relevant information resides and what inductive bias the model imposes are more crucial questions than whether the input is complex. Phase-dominant tasks reward phase-aware views, amplitude-dominant tasks reward magnitude views, and phase-amplitude-coupled tasks promote joint representations across synthetic RF, quantum wavefunctions, and EEG analytic signals. Although they do not dominate explicit real coordinate systems in all regimes, native complex models are useful when their algebraic structure corresponds to this geometry.

The structural reason is that complex multiplication is a constrained real operation. A complex one-dimensional convolution is equivalent to a stacked-real two-channel convolution whose kernel taps are restricted to the $aI+bJ$ subspace, the $U(1)$-equivariant part of the full real channel-mixing space. This constraint can improve generalization when the task respects phase geometry, but it does not provide extra expressivity. It is a symmetry-aligned reduction of expressivity.

The RadioML study adds a second lesson: an evaluation protocol can change the apparent size of a CVNN advantage. On a controlled 3-class RadioML 2018.01A subset, a 22.94 percentage-point CReLU gap under matched-shared-trial selection collapses to 2.46 points under independent per-family tuning. Gradient telemetry and a learning-rate/activation factorial trace the larger gap to high-learning-rate real-baseline failures rather than to a stable architectural advantage.

The resulting recommendation is simple. CVNN-vs-real comparisons should report coordinate-view baselines, capacity and compute controls, both matched-shared-trial and independent per-family selection, and instability diagnostics when gaps are large. Complex-valued models are valuable tools for phase- and amplitude-bearing data, but their advantages are conditional. They should be interpreted through representation, symmetry, and optimization rather than through the datatype alone.


\bibliographystyle{unsrt}  
\bibliography{references}

\end{document}